\documentclass[twoside]{article}

\usepackage[accepted]{aistats2022}
\usepackage[round]{natbib}

\usepackage{amsmath,amsthm,amstext,amsfonts,amssymb,mathrsfs}
\usepackage{algorithm, algorithmic}
\usepackage{enumitem}
\usepackage{graphicx}
\usepackage{textcomp}
\usepackage{xcolor}
\usepackage{bbm}
\usepackage{caption}
\usepackage{picture,xcolor}
\usepackage{url}

\usepackage{booktabs}

\usepackage[colorlinks=true,linkcolor=blue,citecolor=blue,pagebackref=true]{hyperref}       

\newtheorem{lemma}{Lemma}

\newtheorem{proposition}{Proposition}
\newtheorem{theorem}{Theorem}
\newtheorem{definition}{Definition}

\theoremstyle{remark} 
\newtheorem{remark}{Remark}

\newcommand{\E}{\mathbb{E}}
\newcommand{\pr}{\mathbb{P}}
\newcommand{\R}{\mathbb{R}}
\newcommand{\Ss}{\mathcal{S}} 
\newcommand{\Aa}{\mathcal{A}}
\newcommand{\ScA}{\Ss\times\Aa}
\newcommand{\Pp}{\mathcal{P}}

\newcommand{\V}{\mathcal{V}}
\newcommand{\N}{\mathcal{N}}

\newcommand{\cO}{\mathcal{O}}

\renewcommand{\epsilon}{\varepsilon}

\newcommand{\abs}[1]{\left| #1 \right|}%
\newcommand{\norm}[1]{\left\| #1 \right\|}%

\DeclareMathOperator*{\argmax}{arg\,max}

\begin{document}

\runningtitle{Sample Complexity of  Robust Reinforcement Learning with a Generative Model}

\twocolumn[

\aistatstitle{Sample Complexity of  Robust Reinforcement Learning \\with a Generative Model}

\aistatsauthor{ Kishan Panaganti \And Dileep Kalathil }

\aistatsaddress{ Texas A\&M University \And Texas A\&M University } ]

\begin{abstract}
The  Robust Markov Decision Process (RMDP) framework focuses on designing control policies that are robust  against the parameter uncertainties due to the mismatches between the simulator model and real-world settings. An RMDP problem is  typically formulated as a max-min problem, where the objective is to  find the policy that  maximizes the value function for the worst possible model that lies in an uncertainty set around a nominal model.  The standard robust dynamic programming approach requires the knowledge of the nominal model for computing the optimal robust policy. In this work, we propose a model-based reinforcement learning (RL) algorithm for learning an $\epsilon$-optimal robust policy when the nominal model is unknown.  We consider three different forms of uncertainty sets, characterized by the total variation distance, chi-square divergence, and KL divergence. For each of these uncertainty sets, we give a precise characterization of the sample complexity of  our proposed algorithm. In addition to the sample complexity results, we also present a formal analytical argument on the benefit of   using robust policies.  Finally, we demonstrate the  performance of our algorithm on two benchmark problems. 
\end{abstract}



\section{Introduction}

Reinforcement Learning (RL) algorithms typically require a large number of data samples to learn a control policy, which  makes the training of RL algorithms directly on the real-world systems expensive and potentially dangerous. This problem is typically avoided by training the RL algorithm on a simulator and transferring the trained policy to the real-world system.  However, due to multiple reasons such as  the approximation errors incurred while modeling,  changes in the real-world parameters over time and  possible adversarial disturbances in the real-world, there will be inevitable mismatches between the simulator model  and the real-world system. For example, the standard simulator  settings of the sensor noise,  action delays, friction, and  mass of a mobile robot can be different from that of the actual real-world robot. This mismatch between the simulator and real-world  model parameters,  often called `simulation-to-reality gap', can significantly degrade the real-world performance of the RL  algorithms trained on a simulator model.

Robust Markov Decision Process (RMDP)   addresses the \textit{planning} problem of computing the optimal policy  that is robust against the parameter mismatch between the simulator  and real-world  system. The RMDP framework was first introduced in   \citep{iyengar2005robust, nilim2005robust}. The RMDP problem has been analyzed extensively  in the literature \citep{xu2010distributionally, wiesemann2013robust, yu2015distributionally,   mannor2016robust, russel2019beyond}, considering  different types of uncertainty set and computationally efficient algorithms.   However, these works are limited to the {planning} problem, which assumes  the knowledge of the system. Robust RL algorithms for learning the optimal robust policy have also been proposed \citep{roy2017reinforcement,panaganti2020robust}, but they only provide asymptotic convergence guarantees. Robust RL problem has also been addressed using deep RL methods \citep{pinto2017robust, derman2018soft, derman2020bayesian, Mankowitz2020Robust, zhang2020robust}. However, these works are empirical in nature and do not provide any theoretical guarantees for the learned policies. In particular, there are few works that provide \textit{robust RL algorithms with  provable (non-asymptotic)  finite-sample performance guarantees.}

In this work, we address the  problem of developing a model-based robust RL algorithm with provable finite-sample guarantees on its performance, characterized by the metric of sample complexity in a PAC (probably approximately correct) sense. The RMDP framework assumes  that the real-world model  lies within some uncertainty set $\Pp$ around a nominal  (simulator) model $P^{o}$. The goal is to learn a policy that performs the best under the worst possible model in this uncertainty set.  We do not assume that the algorithm knows the exact simulator model (and hence the exact uncertainty set). Instead, similar to the standard (non-robust) RL setting  \citep{singh1994upper, AzarMK13, haskell2016empirical, sidford2018near, agarwal2020model,li2020breaking, kalathil2021empirical}, we assume that the algorithm  has access to a generative sampling model that can generate next-state samples for all state-action pairs according to the nominal simulator model. In this context,  we answer the following important question: \textit{How many samples from the nominal simulator model  do we need to learn an $\epsilon$-optimal robust policy with high probability?}

\begin{table*}[t]
\caption{Comparison of the sample complexities of different uncertainty sets and the best known result in the non-robust setting \citep{li2020breaking}. Here $|\Ss|$ and $|\Aa|$ are the cardinality of the state and action spaces, $c_r$ is the robust RL problem parameter, and $\gamma$ is the discount factor. We consider the optimality gap $\epsilon \in (0, c/(1-\gamma))$, where $c>0$ is a constant.  We refer to Section \ref{subsec:sample-complexity} for further details.}
\label{tab:compare}
   \begin{center}
\begin{sc}
\begin{tabular}{lccccr}
\toprule
Uncertainty set & TV & Chi-square & KL & Non-robust \vspace{0.2cm} \\ 
Sample Complexity & $\cO( \frac{|\Ss|^{2}  |\Aa|}{(1-\gamma)^4 \epsilon^2})$ & $\cO( \frac{c_r |\Ss|^{2}  |\Aa|}{(1-\gamma)^4 \epsilon^2})$ & $\cO( \frac{|\Ss|^{2}  |\Aa| \exp(1/(1-\gamma))}{c_r^2 (1-\gamma)^4 \epsilon^2})$ & $\cO( \frac{|\Ss|  |\Aa|}{(1-\gamma)^3 \epsilon^2})$\\
\bottomrule
\end{tabular}
\end{sc}
\end{center}
    \vspace{-0.1in}
\end{table*}

\textbf{Our contributions:} The main contributions of our work are as follows:

$(1)$  We propose a model-based robust RL algorithm, which we call the robust empirical value iteration algorithm (REVI), for learning an approximately optimal robust  policy.  We consider three different classes of RMDPs with  three different uncertainty sets: $(i)$ Total Variation (TV) uncertainty set, $(ii)$ Chi-square uncertainty set, and $(iii)$ Kullback-Leibler (KL)  uncertainty set, each  characterized by its namesake distance measure.  Robust RL problem is much more challenging than the standard (non-robust) RL problems due to the inherent nonlinearity associated with the robust dynamic programming and the resulting unbiasedness of the empirical estimates. We overcome this challenge analytically by developing a series of upperbounds that are amenable to using concentration inequality results (which are typically useful only in the  unbiased setting), where  we exploit  a uniform concentration bound with a covering number argument. We rigorously characterize the sample complexity of the proposed algorithm for each of these uncertainty sets. We also make a precise comparison with the sample complexity of non-robust RL.  

$(2)$ We  give a formal argument for the need for using a robust policy when the simulator model is different from the  real-world model. More precisely, we analytically address the question \textit{`why do we need robust policies?'}, by showing that the worst case performance of a non-robust policy can be arbitrarily bad (as bad as a random policy) when compared to that of a robust policy. While the need for robust policies have been discussed in the literature qualitatively, to the best of our knowledge, this is the first work that gives an analytical answer to the above question.  

$(3)$ Finally, we demonstrate the performance of our REVI algorithm in two experiment settings and for two different uncertainty sets.  In each setting, we show that the  policy learned by our proposed REVI algorithm is indeed robust against the changes in the model parameters. We also illustrate the convergence of our algorithm with respect to the number of samples and the number of iterations.

\subsection{Related Work}

\textbf{Robust RL:} An RMDP setting  where  some state-action pairs are adversarial and the others  are stationary was considered by \citep{lim2013reinforcement}, who proposed an online algorithm to address this problem. An approximate robust dynamic programming approach with linear function approximation was proposed in  \citep{tamar2014scaling}. State aggregation and kernel-based function approximation for robust RL were studied  in \citep{petrik2014raam, lim2019kernel}.  \citep{roy2017reinforcement} proposed a robust version of  the Q-learning algorithm. \citep{panaganti2020robust} developed a least squares policy iteration approach to learn the optimal robust policy using linear function approximation with provable guarantees. A soft robust RL algorithm was proposed in  \citep{derman2018soft} and a maximum a posteriori policy optimization approach was used in \citep{Mankowitz2020Robust}.  While the above mentioned works make interesting contributions to the area of robust RL, they focus either on giving asymptotic performance guarantees or on the empirical performance without giving provable guarantees. In particular, they do not provide provable guarantees  on the  finite-sample performance  of the robust RL algorithms.


The closest to our work is \citep{zhou2021finite}, which  analyzed the sample complexity with a KL  uncertainty set. Our work is different in two significant aspects: Firstly, we consider the total variation uncertainty set and chi-square uncertainty set, in addition to the KL uncertainty set. The analysis for these uncertainty sets are very challenging and significantly  different from that of the  KL uncertainty set. Secondly, we give a more precise characterization of the sample complexity bound for the  KL uncertainty set by clearly specifying the exponential dependence on $(1-\gamma)^{-1}$, where $\gamma$ is the discount factor, which was left unspecified in  \citep{zhou2021finite}. 

While this paper was under review, we were notified of a concurrent work \citep{yang2021towards}, which also provides similar sample complexity bounds for robust RL. Our proof technique is significantly  different from their work. Moreover,  we also provide open-source experimental results that illustrate the performance of our robust RL algorithm.

\textbf{Other related works:} Robust control is a well-studied area \citep{zhou1996robust, dullerud2013course} in the classical control theory. Recently, there are some interesting works that address the robust RL  problem using this framework, especially focusing on the linear quadratic regulator setting \citep{zhang2020policy}. Our framework of robust MDP is significantly different from this line of work. Risk sensitive RL algorithms \citep{borkar2002q} and adversarial RL algorithms \citep{pinto2017robust}  also address the robustness problem implicitly. Our approach is different from these works also.

\textbf{Notations:} For any set $\mathcal{X}$, $|\mathcal{X}|$ denotes its cardinality. For any vector $x$, $\|x\|$ denotes its infinity norm $\|x\|_{\infty}$.

\section{Preliminaries: Robust Markov Decision Process}
\label{sec:formulation}

A Markov Decision Process (MDP) is a tuple $(\Ss, \Aa, r, P, \gamma)$, where $\Ss$ is the  state space, $\Aa$ is the  action space,  $r: \Ss\times \Aa\rightarrow \R$ is the reward function, and $\gamma \in (0, 1)$ is the discount factor. The transition probability function $P_{s,a}(s')$ represents the probability of transitioning to state $s'$ when  action $a$ is taken at state $s$. $P$ is also called the  the model of the system.   We consider a finite MDP setting where  $|\Ss|$ and $|\Aa|$ are finite. We will also assume that $r(s,a) \in [0,1]$, for all $(s,a) \in \ScA$, without loss of generality.

A  (deterministic)  policy $\pi$ maps each state to an action. The value of a policy $\pi$ for an MDP with model $P$, evaluated at state $s$ is given by
\begin{align}
\label{eq:non-robust-Vpi}
V_{\pi, P}(s) = \E_{\pi, P}[\sum^{\infty}_{t=0} \gamma^{t} r(s_{t}, a_{t}) ~|~ s_{0} = s], \end{align}
where $a_{t} = \pi(s_{t}), s_{t+1} \sim P_{s_t,a_t}(\cdot)$. The optimal  value function $V^{*}_{P}$ and the  optimal policy  $\pi^{*}_{P}$ of an MDP with the model $P$ are defined as 
\begin{align}
\label{eq:non-robust-optimal}
    V^{*}_{P} = \max_{\pi} V_{\pi, P},\quad \pi^{*}_{P} = \argmax_{\pi} V_{\pi, P}. 
\end{align}


\textbf{Uncertainty set:} Unlike the standard  MDP  which considers a single model (transition probability function),  the  RMDP formulation considers a set of models. We call this set as the \textit{uncertainty set} and denote it as $\Pp$.  We  assume that the set $\Pp$ satisfies the standard \textit{rectangularity condition} \citep{iyengar2005robust}.   We note that a similar {uncertainty set} can be considered for the reward function at the expense of additional notations. However, since the analysis will be similar and the samples complexity guarantee will be identical upto a constant, without loss of generality,  we assume that the reward function  is known and deterministic.  We specify a robust MDP as a tuple $M = (\Ss, \Aa, r, \Pp, \gamma)$.

The uncertainty set $\mathcal{P}$ is typically defined as
\begin{align}
    \label{eq:uncertainty-set}
    \mathcal{P} = \otimes\, \mathcal{P}_{s,a},~\text{where,}~
    \mathcal{P}_{s,a} =  \{ P_{s,a} \in [0,1]^{|\Ss|} ~:~ \nonumber \\ D(P_{s,a}, P^o_{s,a}) \leq c_{r}, \sum_{s'\in\Ss} P_{s,a}(s') = 1  \},
\end{align}
where  $P^{o} = (P^o_{s,a}, (s, a) \in \ScA)$ is the nominal transition probability function, $c_r>0$ indicates the level of robustness,  and $D(\cdot, \cdot)$ is a distance metric between two probability distributions. In the following, we call $P^{o}$ as the nominal model.  In other words, $\Pp$ is the set of all valid transition probability  functions in the neighborhood of the nominal model $P^{o}$, where the neighborhood is defined using the distance metric $D(\cdot, \cdot)$.  We note that the radius $c_{r}$ can depend on the state-action pair $(s,a)$. We omit this to reduce the notation complexity. We also note that for $c_r\downarrow 0$, we recover the non-robust regime.

We consider three different uncertainty sets corresponding to three different distance metrics $D(\cdot, \cdot)$.  

1. \textit{Total Variation (TV) uncertainty set ($\mathcal{P}^{\mathrm{tv}} $)}: We define $\mathcal{P}^{\mathrm{tv}}  = \otimes \mathcal{P}^{\mathrm{tv}}_{s,a}$, where  $\mathcal{P}^{\mathrm{tv}}_{s,a}$ is defined as in \eqref{eq:uncertainty-set} using the total variation distance 
\begin{align}
\label{eq:D-L1}
D_{\mathrm{tv}}(P_{s,a}, P^{o}_{s,a}) = (1/2) \|P_{s,a} - P^{o}_{s,a} \|_{1}.
\end{align}
2.  \textit{Chi-square uncertainty set ($\mathcal{P}^{\mathrm{c}}$)}:   We define $\mathcal{P}^{\mathrm{c}}  = \otimes \mathcal{P}^{\mathrm{c}}_{s,a}$, where  $\mathcal{P}^{\mathrm{c}}_{s,a}$ is defined as in \eqref{eq:uncertainty-set} using the Chi-square distance 
\begin{align}
\label{eq:D-Chi}
D_{\mathrm{c}}(P_{s,a}, P^{o}_{s,a}) =  \sum_{s'\in\Ss} \frac{(P_{s,a}(s') - P^{o}_{s,a}(s'))^2}{P^{o}_{s,a}(s')}.
\end{align}
3. \textit{Kullback-Leibler (KL) uncertainty set ($\mathcal{P}^{\mathrm{kl}}$)}:   We define $\mathcal{P}^{\mathrm{kl}}  = \otimes \mathcal{P}^{\mathrm{kl}}_{s,a}$, where  $\mathcal{P}^{\mathrm{kl}}_{s,a}$ is defined as in \eqref{eq:uncertainty-set} using the Kullback-Leibler (KL) distance
\begin{align}
\label{eq:D-KL}
D_{\mathrm{kl}}(P_{s,a}, P^{o}_{s,a})  = \sum_{s'} P_{s,a}(s') \log \frac{P_{s,a}(s')}{P^{o}_{s,a}(s')}.
\end{align}
We note that the sample complexity  and its analysis  will  depend on the specific form of the uncertainty set.

\textbf{Robust value iteration:} The goal of the RMDP problem is to compute the optimal  robust policy  which  maximizes the value  even under the worst model in the uncertainty set.  Formally, the \textit{robust value function} $V_{\pi}$ corresponding to a policy $\pi$ and the \textit{optimal robust value function} $V^{*}$ are defined as \citep{iyengar2005robust,nilim2005robust}
\begin{align}
V^{\pi} = \inf_{P \in \mathcal{P}} ~V_{\pi, P},\quad V^{*} = \sup_{\pi} \inf_{P \in \mathcal{P}} ~V_{\pi, P} . 
\end{align} 
The \textit{optimal robust policy} $\pi^{*}$ is such that the robust value function corresponding to it matches the optimal robust value function, i.e., $V^{\pi^*} =V^{*} $. It is known that there exists a deterministic optimal policy \citep{iyengar2005robust} for the RMDP problem. So, we will restrict our attention to the class of deterministic policies. 

For any set  $\mathcal{B}$ and a vector $v$, let 
\begin{align*}
     \sigma_{\mathcal{B}} (v) = \inf \{ u^{\top}v : u \in \mathcal{B} \}.
\end{align*}
Using this  notation, we can define the \textit{robust Bellman operator} \citep{iyengar2005robust} as 
$
T (V) (s) =  \max_{a}\, ( r(s,a) + \gamma \sigma_{\mathcal{P}_{s,a}}(V)). 
$
It is known that $T$ is a contraction mapping in infinity norm and the  $V^{*}$ is the unique fixed point of $T$  \citep{iyengar2005robust}.  Since $T$ is a contraction, \textit{robust value iteration} can be used to compute $V^{*}$,  similar to the non-robust MDP setting  \citep{iyengar2005robust}. More precisely, the robust  value iteration, defined as  $V_{k+1} = T V_{k}$, converges to $V^{*}$, i.e., $V_{k} \rightarrow V^{*}$. Similar to the optimal robust value function, we can also define the optimal robust action-value function as $
    Q^{*}(s,a) = r(s, a) + \gamma \sigma_{\mathcal{P}_{s,a}}(V^{*})$.
Similar to the non-robust setting, it is straight forward to show that $\pi^{*}(s) = \argmax_{a} Q^{*}(s,a)$ and $V^{*}(s) = \max_{a} Q^{*}(s,a)$.

\section{Algorithm and Sample Complexity}
\label{sec:algorithm-results}

The robust  value iteration requires the knowledge of the nominal model $P^{o}$ and the radius of the uncertainty set  $c_{r}$ to  compute  $V^{*}$ and $\pi^{*}$. While $c_{r}$ may be  available as  design parameter, the  form of the nominal model may not be available in most practical problems. So,  we do not assume the knowledge of  the nominal model $P^{o}$.  Instead, similar to the non-robust RL setting, we assume only to have access to the samples from a generative model,  which can generate samples of the next state $s'$  according to $P^{o}_{s,a}(\cdot)$, given the state-action pair $(s,a)$ as the input. We propose a model-based robust RL algorithm that uses these samples to estimate the nominal model and uncertainty set.   

\subsection{Robust Empirical Value Iteration (REVI) Algorithm}

We first get a maximum likelihood estimate $\widehat{P}^{o}$ of the nominal model $P^{o}$ by following the standard approach  \cite[Algorithm 3]{AzarMK13}. More precisely, we generate $N$  next-state samples corresponding to  each state-action pairs. Then, the   maximum likelihood estimate  $\widehat{P}^{o}$ is given by $\widehat{P}^{o}_{s,a}(s') = {N(s,a,s')}/{N}$, where $N(s,a,s')$  is the number of times the state $s'$ is realized out of the total $N$  transitions from the state-action pair $(s, a)$. Given $\widehat{P}^{o}$, we can get  an empirical estimate $\widehat{\Pp}$ of the uncertainty set $\Pp$ as, 
\begin{align}
    \label{eq:uncertainty-set-estimate}
    \widehat{\Pp} = &\otimes\, \widehat{\mathcal{P}}_{s,a},~\text{where},~~\widehat{\Pp}_{s,a} = \{ P\in [0,1]^{\Ss}~:~ \nonumber  \\
    &D( P_{s,a} , \widehat{P}_{s,a}) \leq c_r, \sum_{s'\in\Ss} P_{s,a}(s') = 1  \},
\end{align} 
where $D$ is one of the metrics  specified in \eqref{eq:D-L1} - \eqref{eq:D-KL}.

For finding an approximately optimal robust policy, we now  consider  the empirical RMDP  $\widehat{M} = (\Ss, \Aa, r, \widehat{\Pp}, \gamma)$ and perform robust value iteration using $\widehat{\Pp}$. This is indeed our approach, which we call the Robust Empirical Value Iteration (REVI) Algorithm.  The optimal robust policy and  value function of $\widehat{M}$ are denoted as $\widehat{\pi}^{\star}, \widehat{V}^{\star}$, respectively. 

\begin{algorithm}[t!]
	\caption{Robust Empirical Value Iteration (REVI) Algorithm}	
	\label{revialgo}
	\begin{algorithmic}[1]
		\STATE \textbf{Input:} Loop termination number $K$
		\STATE \textbf{Initialize:} $Q_0 = 0$
		\STATE Compute the empirical uncertainty set  $\widehat{\Pp}$ according to  \eqref{eq:uncertainty-set-estimate}
		\FOR {$k=0,\cdots,K-1$ }
		\STATE $V_{k}(s) = \max_{a} Q_{k}(s, a),~\forall s$
		\STATE $Q_{k+1}(s,a) = r(s,a) +  \gamma \sigma_{\widehat{\Pp}_{s,a}} (V_{k}), ~\forall (s, a)$ 
		\ENDFOR
		
		\STATE \textbf{Output:} $\pi_K(s) = \argmax_a Q_{K}(s,a), \forall s \in \Ss$
	\end{algorithmic}
\end{algorithm}

\subsection{Sample Complexity}\label{subsec:sample-complexity}

In this section we give the sample complexity guarantee  of the REVI algorithm  for the three uncertainty sets. We first consider the TV uncertainty set. 
\begin{theorem}[TV Uncertainty Set]
\label{thm:revi-TV-guarantee}
Consider an RMDP with  a total variation uncertainty set $\mathcal{P}^{\mathrm{tv}}$.  Fix $\delta\in(0,1)$ and $\epsilon\in(0,24\gamma /(1-\gamma))$. Consider the REVI algorithm with $K \geq K_{0}$ and $N \geq N^{\mathrm{tv}}$, where
	\begin{align}
	\label{eq:K0}
	    K_{0} &=   \frac{1}{ \log(1/\gamma)}  \log(\frac{8\gamma}{\epsilon (1-\gamma)^2}) \text{ and }   \\
	  \label{eq:N-TV}
	    N^{\mathrm{tv}} &= \frac{ 72 \gamma^2   |\Ss|}{(1-\gamma)^4 \epsilon^2}  \log(\frac{144\gamma |\Ss| |\Aa|}{(\delta\epsilon(1-\gamma)^2)}).
	\end{align}
Then,  $ \| V^* - V^{\pi_K} \| \leq \epsilon$ with probability at least $1-2\delta$. 
\end{theorem}

\begin{remark}
The total number of samples needed in the REVI algorithm is $N_{\mathrm{total}} = N |\Ss| |\Aa|$. So the sample complexity of the  REVI algorithm with the TV uncertainty set is $\cO( \frac{|\Ss|^{2}  |\Aa|}{(1-\gamma)^4 \epsilon^2})$. 
\end{remark}

\begin{remark}[Comparison with the sample complexity of the non-robust RL]
For the non-robust setting, the lowerbound for the total number of samples from the  generative sampling device is  $\Omega(\frac{|\Ss| |\Aa|}{\epsilon^{2} (1-\gamma)^{3}} \log \frac{|\Ss| |\Aa|}{\delta })$  \citep[Theorem 3]{AzarMK13}. 
The variance reduced value iteration algorithm proposed in \citep{sidford2018near} achieves a  sample complexity of  $\cO(\frac{|\Ss| |\Aa|}{\epsilon^{2} (1-\gamma)^{3}} \log \frac{|\Ss| |\Aa|}{\delta \epsilon })$, matching the lower bound. However, this work is restricted to $\epsilon \in (0, 1)$, whereas $\epsilon$ can be considered upto the value $1/(1-\gamma)$ for the MDP problems.  Recently, this result has been further improved recently by \citep{agarwal2020model}  and \citep{li2020breaking}, which considered  $\epsilon \in (0, 1/\sqrt{(1-\gamma)})$ and  $\epsilon \in (0, 1/{(1-\gamma)})$, respectively. 

Theorem \ref{thm:revi-TV-guarantee} for the robust RL setting also considers $\epsilon$ upto $\cO(1/(1-\gamma))$. However,  the sample complexity  obtained   is worse by a factor of $|\Ss|$ and $1/(1-\gamma)$ when compared to the non-robust setting.  These additional terms are appearing in our result due to a covering number argument we used in the proof, which seems necessary for getting a tractable bound. However, it is not  clear if this is fundamental to the robust RL problem with TV uncertainty set. We leave this investigation for our future work. 
\end{remark}

 We next consider the chi-square  uncertainty set. 
\begin{theorem}[Chi-square Uncertainty Set]
\label{thm:revi-Chi-guarantee}
Consider an RMDP with  a Chi-square uncertainty set $\mathcal{P}^{\mathrm{c}}$.  Fix $\delta \in (0,1)$ and $\epsilon\in(0,16\gamma /(1-\gamma))$, for an absolute constant $c_1>1$. Consider the REVI algorithm with $K \geq K_{0}$ and $N \geq N^{\mathrm{c}}$, where $K_{0}$ is as given in \eqref{eq:K0} and 
	\begin{align}
	  \label{eq:N-Chi}
	    N^{\mathrm{c}} &= \frac{  64 \gamma^2 (2c_r + 1)|\Ss|}{(1-\gamma)^4 \epsilon^2}  \log(\frac{192 |\Ss|  |\Aa|\gamma}{(\delta\epsilon(1-\gamma)^2)}).
	\end{align}
Then,  $ \| V^* - V^{\pi_K} \| \leq \epsilon$ with probability at least $1-2\delta$. 
\end{theorem} 

\vspace{0.1cm}
\begin{remark}
The sample complexity of the  algorithm with the chi-square  uncertainty set is $\cO( \frac{|\Ss|^{2}  |\Aa| c_r}{(1-\gamma)^4 \epsilon^2})$. The order of sample complexity remains the same compared to that of the TV uncertainty set given in Theorem \ref{thm:revi-TV-guarantee}. 
\end{remark}

Finally, we consider the KL  uncertainty set. 
\begin{theorem}[KL Uncertainty Set]
\label{thm:revi-KL-guarantee}
Consider an RMDP with  a KL uncertainty set $\mathcal{P}^{\mathrm{kl}}$.  Fix $\delta \in (0,1)$ and $\epsilon\in(0, 1/(1-\gamma))$.  Consider the REVI algorithm with $K \geq K_{0}$ and $N \geq N^{\mathrm{kl}}$, where $K_{0}$ is as in \eqref{eq:K0} and 
\begin{align}
\label{eq:N-KL}
&\hspace{-0.2cm}N^{\mathrm{kl}} \hspace{-0.1cm}= \frac{8\gamma^2|\Ss|}{c_r^2(1-\gamma)^4 \epsilon^{2}}  \exp(\frac{2\lambda_{\mathrm{kl}}+4}{\lambda_{\mathrm{kl}}(1-\gamma)}) \log(\frac{9 |\Ss|  |\Aa|}{\delta\lambda_{\mathrm{kl}}(1-\gamma)}),
\end{align}
and $\lambda_{\mathrm{kl}}$ is a problem dependent parameter but independent of $N^{\mathrm{kl}}$. 
Then,  $ \| V^* - V^{\pi_K} \| \leq \epsilon$ with probability at least $1-2\delta$. 
\end{theorem}

\vspace{0.1cm}
\begin{remark}
The sample complexity with the KL uncertainty set is $\cO( \frac{|\Ss|^2  |\Aa|}{(1-\gamma)^{4} \epsilon^{2} c_r^2 } \exp(\frac{1}{(1-\gamma)}))$. We note that \citep{zhou2021finite} also considered the robust RL problem with KL uncertainty set. They provided a sample complexity bound of the form $\cO( \frac{C |\Ss|^{2}  |\Aa|}{(1-\gamma)^{4} \epsilon^{2} c_r^2 })$, However  the exponential dependence on $1/(1-\gamma)$ was hidden inside  the constant $C$. In this work, we clearly specify the depends on the factor $1/(1-\gamma)$. 
\end{remark}



\section{Why Do We Need Robust Policies?}

In the introduction, we have given a qualitative description about the need for finding a robust policy. In this section, we give a formal argument  to show that the worst case performance of a non-robust policy can be arbitrarily bad (as bad as a random policy) when compared to that of a robust policy.  

We consider a simple setting with an uncertainty set that contains only two models, i.e.,  $\mathcal{P} = \{P^{o}, P'\}$. Let $\pi^{*}$ be the optimal robust policy.  Following the notation in \eqref{eq:non-robust-optimal}, let $\pi^{o} =  \pi_{P^{o}}$ and $\pi' = \pi_{P^{'}}$ be the non-robust optimal policies when the model is $P^{o}$ and $P'$, respectively. Assume that nominal model is $P^{o}$ and we decide to employ the non-robust policy $\pi^{o}$. The worst case performance of $\pi^{o}$ is characterized by its robust value function $V^{\pi^{o}}$ which is $ \min \{V_{\pi^{o}, P^{o}}, V_{\pi^{o}, P'}\}$. 

We now state the following result.
\begin{theorem}[Robustness Gap] 
\label{thm:non-robust-policy-suboptimality}
There exists a robust MDP $M$ with uncertainty set $\mathcal{P} = \{P^{o}, P'\}$, discount factor  $\gamma\in(\gamma_o,1]$, and state $s_{1}\in\Ss$ such that 
\begin{align*}
    V^{{\pi}^{o}}(s_1) \leq V^{\pi^{*}}(s_1) - {c}/{(1-\gamma)},
\end{align*}
where $c$ is a positive constant, $\pi^{*}$ is the optimal robust policy, and $\pi^{o} =  \pi_{P^{o}}$ is the non-robust optimal policy when the model is $P^{o}$.
\end{theorem}
Theorem \ref{thm:non-robust-policy-suboptimality} states that the worst case performance of the non-robust policy $\pi^{o}$ is lower than that of the optimal robust policy $\pi^{*}$, and this performance gap is $\Omega(1/(1-\gamma))$. Since  $|r(s,a)| \leq 1, \forall (s,a) \in \Ss \times \Aa$ by assumption, $\|V_{\pi, P}\| \leq 1/(1-\gamma)$ for any policy $\pi$ and any model $P$. Therefore, the difference between the optimal  (robust) value function and the (robust) value function of an arbitrary policy  cannot be greater than $\cO(1/(1-\gamma))$. Thus the worst-case performance of the non-robust policy $\pi^{o}$ can be as bad as an arbitrary policy in an order sense.

\section{Sample Complexity Analysis}
\label{sec:analysis}
In this section we explain the key ideas used in the  analysis of the REVI algorithm for obtaining the sample complexity bound for each of the uncertainty sets.  Recall that we consider an RMDP ${M}$ and its empirical estimate version as $\widehat{M}$. 


To bound  $\| V^* - V^{\pi_{K}} \|$, we split it into three terms as $\| V^* - V^{\pi_{K}} \| \leq \| V^* - \widehat{V}^{*} \| + \| \widehat{V}^{*} - \widehat{V}^{\pi_{K}} \| + \| \widehat{V}^{\pi_{K}} - V^{\pi_{K}} \|,$
and analyze each term separately. 

Analyzing the second term, $\| \widehat{V}^{*} - \widehat{V}^{\pi_{K}} \|$,  is similar to that  of non-robust algorithms. Due to the contraction property of the robust Bellman operator, it is straight forward to show that $ \| \widehat{V}^{*} - \widehat{V}^{\pi_{k+1}} \| \leq \gamma  \| \widehat{V}^{*} - \widehat{V}^{\pi_{k}} \|$ for any $k$. This exponential convergence, with some additional results from the MDP theory,   enables us to get a bound  $
\| \widehat{V}^{*} - \widehat{V}^{\pi_{K}} \| \leq {2\gamma^{K+1}}/{(1-\gamma)^2}$.

The analysis of terms $\| V^* - \widehat{V}^{*}\|$ and $\| \widehat{V}^{\pi_{K}} - V^{\pi_{K}} \|$ are however  non-trivial and significantly more challenging  compared to the non-robust setting. We will focus on the latter, and the analysis of the former is similar. 

For any policy $\pi$ and for any state $s$, and denoting $a = \pi(s)$,   we have
\small
\begin{align}
&V^{\pi}(s) -  \widehat{V}^{\pi}(s)=\gamma \sigma_{{\Pp}_{s,a}} (V^{\pi}) - \gamma \sigma_{\widehat{\Pp}_{s,a}} (\widehat{V}^{\pi}) \nonumber \\
\label{eq:sigma-split-1}
&= \gamma  (\sigma_{{\Pp}_{s,a}} (V^{\pi}) -  \sigma_{{\Pp}_{s,a}} (\widehat{V}^{\pi})) +  \gamma (\sigma_{{\Pp}_{s,a}} (\widehat{V}^{\pi}) - \sigma_{\widehat{\Pp}_{s,a}} (\widehat{V}^{\pi}) )
\end{align} 
\normalsize
To bound the first term in \eqref{eq:sigma-split-1}, we present a result that shows that $\sigma_{{\Pp_{s,a}}}$   is $1$-Lipschitz  in the sup-norm.
\begin{lemma}
\label{lem:sigma_v_diff}
For any $(s,a)\in\ScA$ and  for any $ V_{1}, V_{2} \in \R^{|\Ss|}$, we have $ |{\sigma}_{{\Pp}_{s,a}} (V_{1}) - {\sigma}_{{\Pp}_{s,a}} (V_{2}) | \leq \|V_{1} - V_{2}\|$ and $
	     |{\sigma}_{\widehat{\Pp}_{s,a}} (V_{1}) - {\sigma}_{\widehat{\Pp}_{s,a}} (V_{2}) | \leq \|V_{1} - V_{2}\|.$
\end{lemma}
Using the above lemma, the first term in  \eqref{eq:sigma-split-1} will be bounded by $\gamma  \|V^{\pi} -  \widehat{V}^{\pi}\|$ and the discount factor makes this term amenable to getting a closed form bound. 

Obtaining a bound for  $\sigma_{{\Pp}_{s,a}} (\widehat{V}^{\pi}) - \sigma_{\widehat{\Pp}_{s,a}} (\widehat{V}^{\pi})$ is the most challenging part of our analysis. In the non-robust setting, this  will be equivalent to the error term $P^{o}_{s,a} V - \widehat{P}_{s,a} V$, which is unbiased and can be easily bounded using  concentration inequalities. In the robust setting, however, because of the  nonlinear nature of the function $\sigma(\cdot)$, $\mathbb{E}[\sigma_{\widehat{\Pp}_{s,a}} (\widehat{V}^{\pi})] \neq \sigma_{{\Pp}_{s,a}} (\widehat{V}^{\pi})$. So, using concentration inequalities to get a bound is not immediate. Our strategy is to find appropriate upperbound for this term that is amenable to using concentration inequalities. To that end, we will analyze this term separately for each of the three uncertainty set. 

\subsection{Total variation uncertainty set}
We will first get following upperbound: 
\begin{lemma}[TV uncertainty set]
\label{lem:tv-sigma-diff}
Let  $\mathcal{V} = \{V \in \mathbb{R}^{|\Ss|}: \norm{V} \leq 1/(1 - \gamma)\}$. For any $(s, a) \in \Ss \times \Aa$ and for any $V \in \V$, 
\begin{align}
\begin{split}
\label{eq:tv-sigma-diff}
 | \sigma_{\widehat{\Pp}^{\mathrm{tv}}_{s,a}} (V) - \sigma_{{\Pp}^{\mathrm{tv}}_{s,a}} (V) | &\leq  2  \max_{\mu \in \V} | \widehat{P}_{s,a} \mu - P^o_{s,a} \mu|.
\end{split}
\end{align}
\end{lemma}
While the term $| \widehat{P}_{s,a} \mu - P^o_{s,a} \mu|$ in \eqref{eq:tv-sigma-diff} can be upperbounded using the standard Hoeffding's inequality, bounding $\max_{\mu\in\V} | \widehat{P}_{s,a} \mu - P^o_{s,a} \mu|$ is more challenging as it requires a uniform bound. Since $\mu$ can take a continuum of values, a simple union bound argument will also not work. We overcome this issue by using a covering number argument and obtain the following bound. 
\begin{lemma}
\label{lem:covering-hoeffdings}
Let $V \in \mathbb{R}^{|\Ss|}$  with $\norm{V} \leq 1/(1-\gamma)$. For any $\eta, \delta \in (0, 1)$, 
\begin{align*}
&\max_{\mu: 0 \leq \mu \leq V}~ \max_{s,a} ~ | \widehat{P}_{s,a} \mu - P^o_{s,a} \mu|  \leq \\
&\hspace{1.5cm} \frac{1}{1-\gamma} \sqrt{\frac{|\Ss|}{2N} \log( \frac{12|\Ss||\Aa|}{(\delta\eta(1-\gamma))} } + 2\eta,
\end{align*}
with probability at least $1 - \delta/2$. 
\end{lemma}
We note that this uniform bound adds an additional $\sqrt{|\Ss|}$ factor compared to the  non-robust setting, which results in an additional ${|\Ss|}$ in the sample complexity.  Combining these, we finally get the following result. 
\begin{proposition}
\label{lem:tv-sigma-diff-bound-uniform}
Let  $\mathcal{V} = \{V \in \mathbb{R}^{|\Ss|}: \norm{V} \leq 1/(1 - \gamma)\}$. For any $\eta, \delta \in (0, 1)$, with probability at least $1 - \delta$,
\begin{align}
\max_{V \in \V} &\max_{s,a}  | \sigma_{\widehat{\Pp}^{\mathrm{tv}}_{s,a}} (V) - \sigma_{{\Pp}^{\mathrm{tv}}_{s,a}} (V) |  \leq   {C}^{\mathrm{tv}}_{u}(N,\eta, \delta),\text{ where,} \nonumber \\
&{C}^{\mathrm{tv}}_{u}(N,\eta, \delta) = 4 \eta ~+ \nonumber\\& \hspace{0.5cm} \frac{2}{1-\gamma} \sqrt{\frac{|\Ss|\log(6|\Ss||\Aa|/(\delta\eta(1-\gamma)))}{2N}}.  \label{eq:c-tv-uniform}
\end{align}
\end{proposition}
Tracing back the steps to \eqref{eq:sigma-split-1}, we can get an arbitrary small bound for  $\|V^{\pi} -  \widehat{V}^{\pi}\|$ by selecting $N$ appropriately, as specified in Theorem \ref{thm:revi-TV-guarantee}. 

\subsection{Chi-square uncertainty set}
We will first get the following upperbound: 
\begin{lemma}[Chi-square uncertainty set]
\label{lem:chi-sigma-diff}
For any $(s, a) \in \Ss \times \Aa$ and for any $V \in \mathbb{R}^{|\Ss|}, \norm{V} \leq 1/(1-\gamma)$, 
\begin{align}
\label{eq:chi-sigma-diff}
 &| \sigma_{\widehat{\Pp}^{\mathrm{c}}_{s,a}} (V) - \sigma_{{\Pp}^{\mathrm{c}}_{s,a}} (V)  |  \leq  \nonumber \\
 & \max_{\mu: 0 \leq \mu \leq V} | \sqrt{c_r \mathrm{Var}_{\widehat{P}_{s,a}}(V-\mu)} - \sqrt{c_r \mathrm{Var}_{P^o_{s,a}}(V-\mu)}| \nonumber  \\
&+ \max_{\mu: 0 \leq \mu \leq V} | \widehat{P}_{s,a} (V-\mu) - P^o_{s,a} (V-\mu)|. 
\end{align}
\end{lemma}
The second term  of \eqref{eq:chi-sigma-diff} can be bounded using Lemma \ref{lem:covering-hoeffdings}.  However, the first term, which involves the square-root of the variance is more challenging. We use a concentration inequality that is applicable for variance to overcome this challenge.
Finally, we get the following result. 
\begin{proposition}
\label{lem:chi-sigma-diff-bound-uniform}
Let  $\mathcal{V} = \{V \in \mathbb{R}^{|\Ss|}: \norm{V} \leq 1/(1 - \gamma)\}$. For any $\eta, \delta \in (0, 1)$, with probability at least $(1 - \delta)$,
\begin{align*}
\max_{V \in \V} \max_{s,a} | \sigma_{\widehat{\Pp}^{\mathrm{c}}_{s,a}} (V) - \sigma_{{\Pp}^{\mathrm{c}}_{s,a}} (V) |  \leq   {C}^{\mathrm{c}}_{u}(N,\eta, \delta),~\text{where,} 
\end{align*}
\begin{align}
&{C}^{\mathrm{c}}_{u}(N,\eta, \delta) \leq  \sqrt{\frac{32\eta c_r}{1-\gamma}}  + 2\eta
+\nonumber\\  &\hspace{0.1cm}\frac{1}{1-\gamma} \sqrt{\frac{(2 c_r+1) |\Ss|\log(12|\Ss||\Aa|/(\delta\eta(1-\gamma)))}{N}} ,
\label{eq:c-chi-uniform}
\end{align}
\end{proposition}
Now, by selecting appropriate $N$ as specified in Theorem \ref{thm:revi-Chi-guarantee}, we can show the $\epsilon$-optimality of $\pi_{K}$. 

The details on the KL uncertainty set analysis is included in the appendix.

\section{Experiments}

\begin{figure*}[ht]
    \centering
	\includegraphics[width=\linewidth]{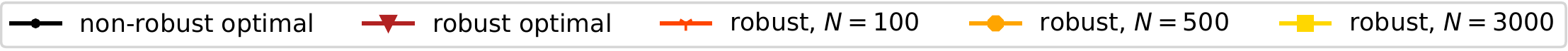}
	\begin{minipage}{.24\textwidth}
		\centering
		\includegraphics[width=\linewidth]{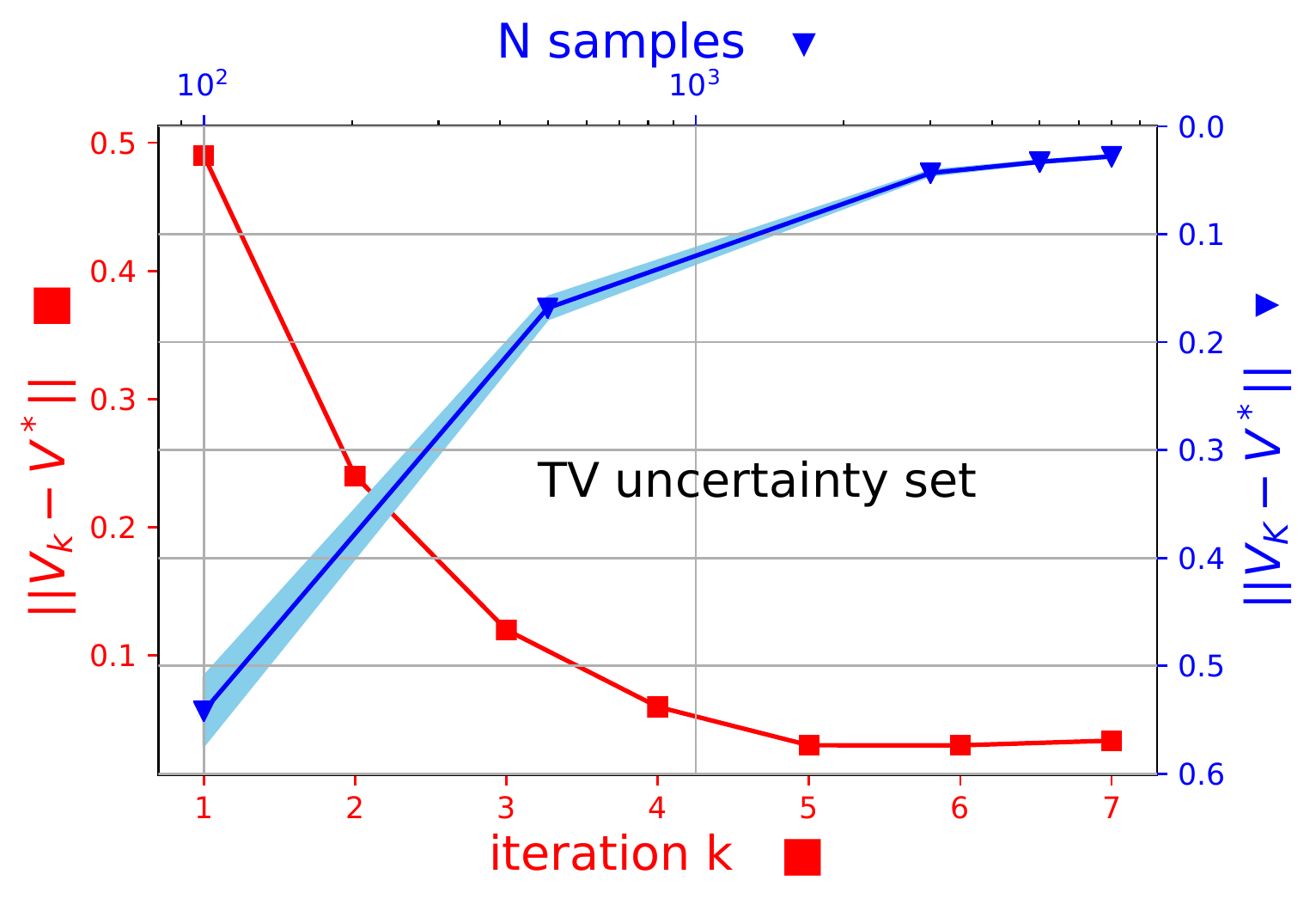}
	\end{minipage}
	\begin{minipage}{.24\textwidth}
		\centering
		\includegraphics[width=\linewidth]{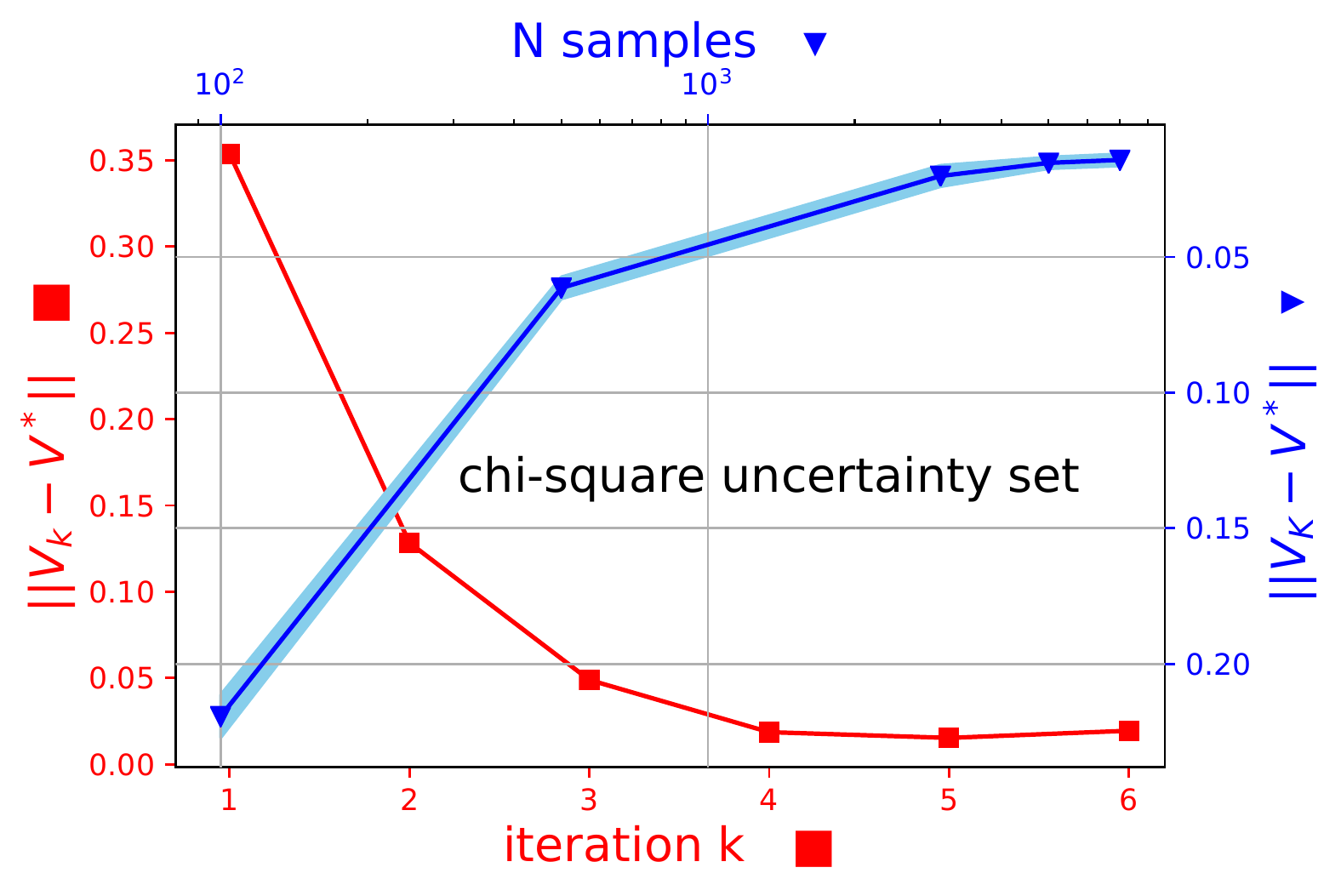}
	\end{minipage}
	\begin{minipage}{.24\textwidth}
		\centering
		\includegraphics[width=\linewidth]{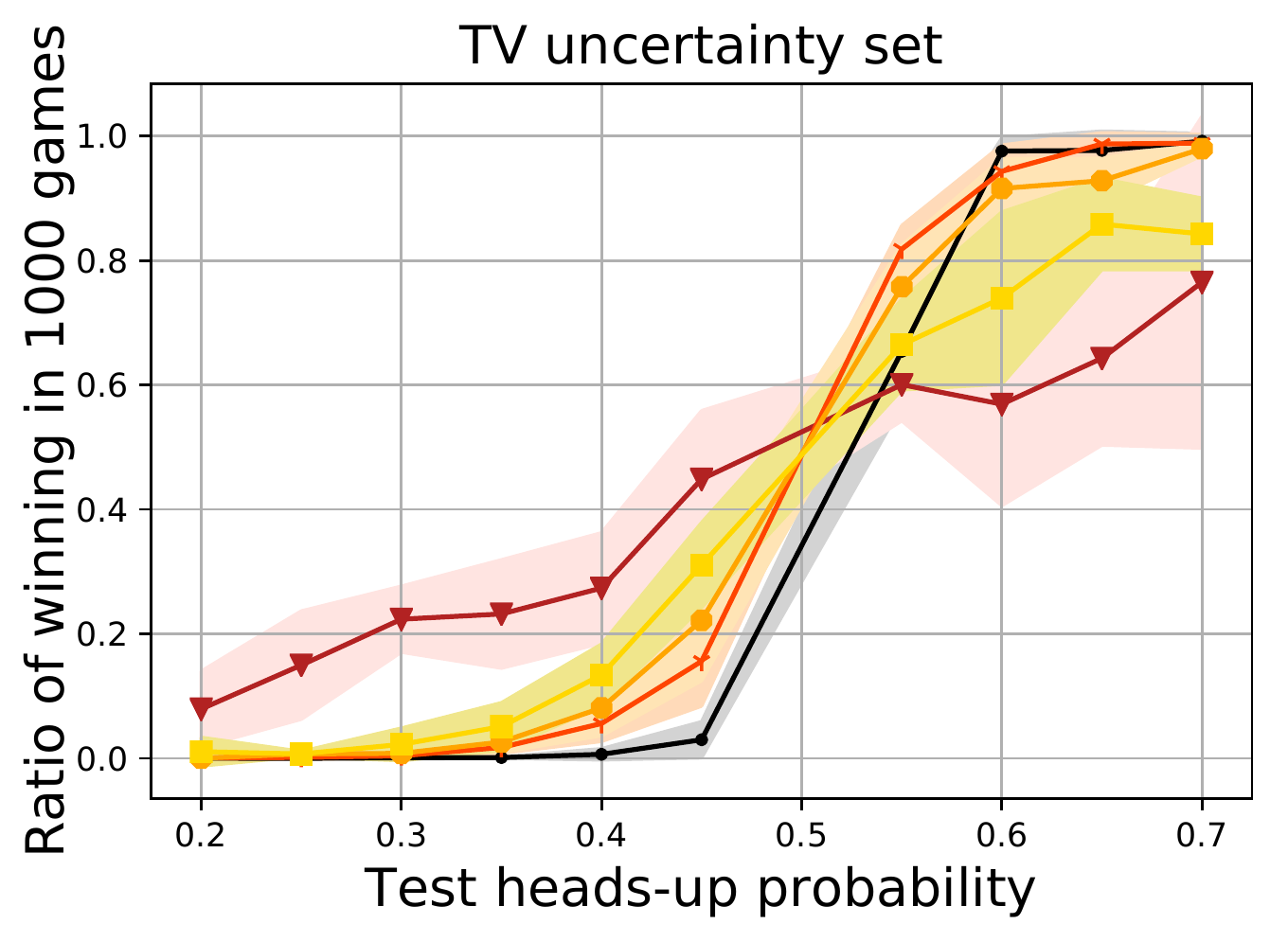}
	\end{minipage}
	\begin{minipage}{.24\textwidth}
		\centering
		\includegraphics[width=\linewidth]{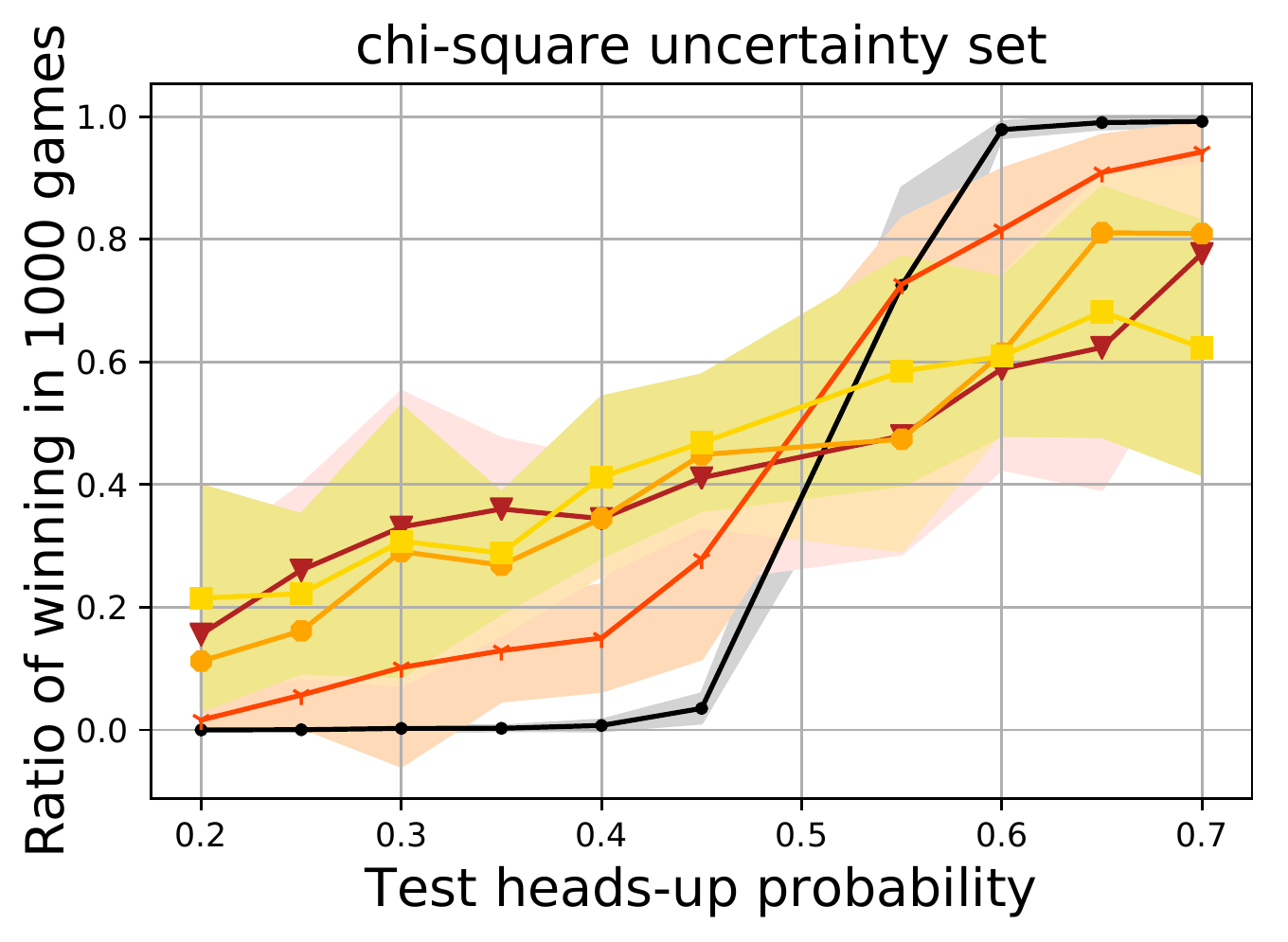}
	\end{minipage}
	\centering
	\captionof{figure}{\textit{Experiment results for the Gambler's problem.} The first two plots shows the rate of convergence with respect to the number of iterations ($k$) and the rate of convergence with respect to the number of samples $(N)$ for the TV and chi-square uncertainty set, respectively. The third and fourth plots shows the robustness of the learned policy against changes in the model parameter (heads-up probability).}
	\label{fig:gamblers}
	
    \centering
	\includegraphics[width=\linewidth]{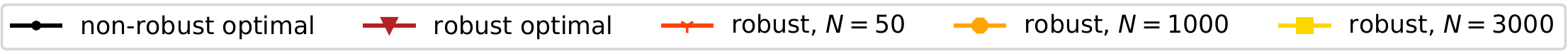}
	\begin{minipage}{.24\textwidth}
		\centering
		\includegraphics[width=\linewidth]{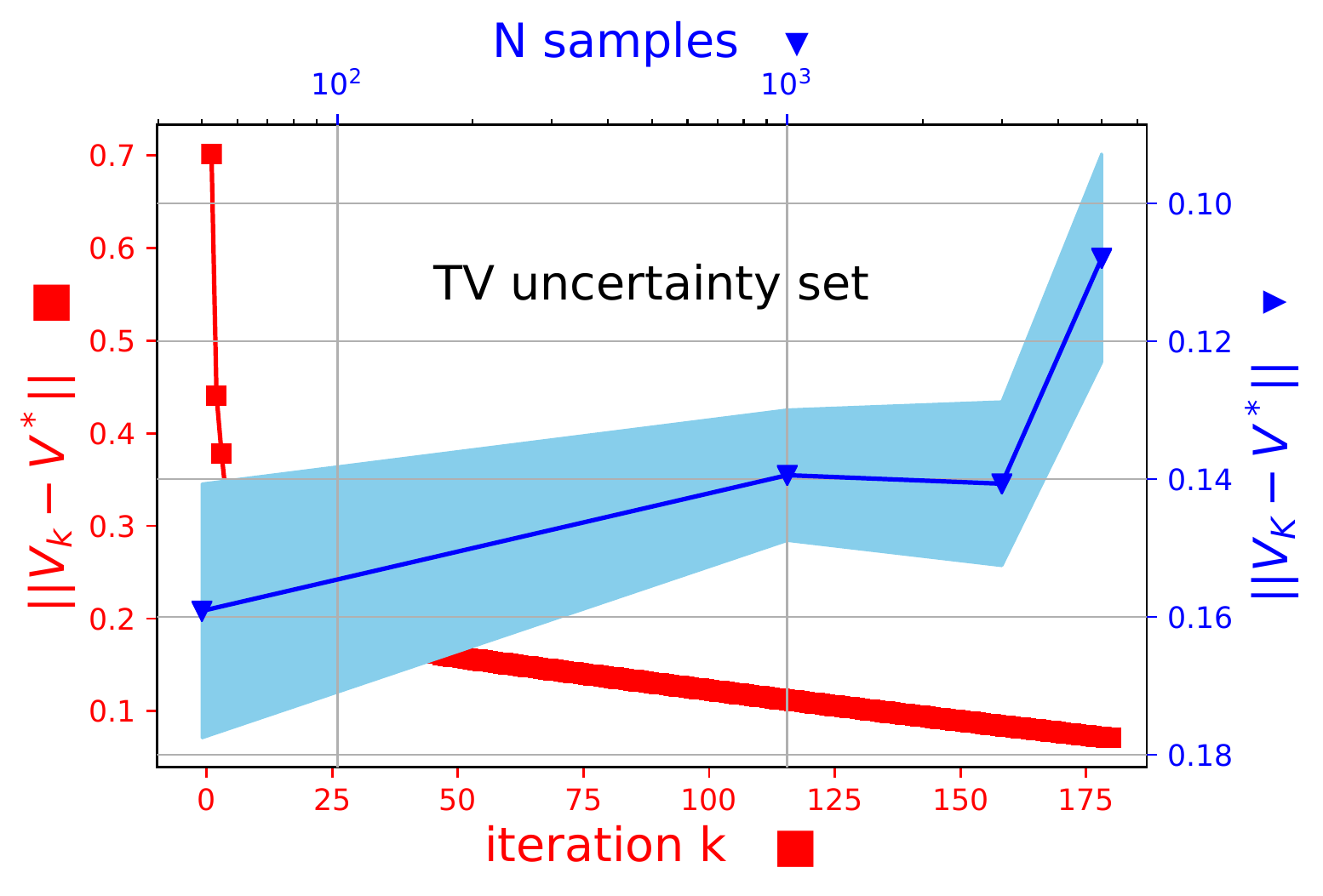}
	\end{minipage}
	\begin{minipage}{.24\textwidth}
		\centering
		\includegraphics[width=\linewidth]{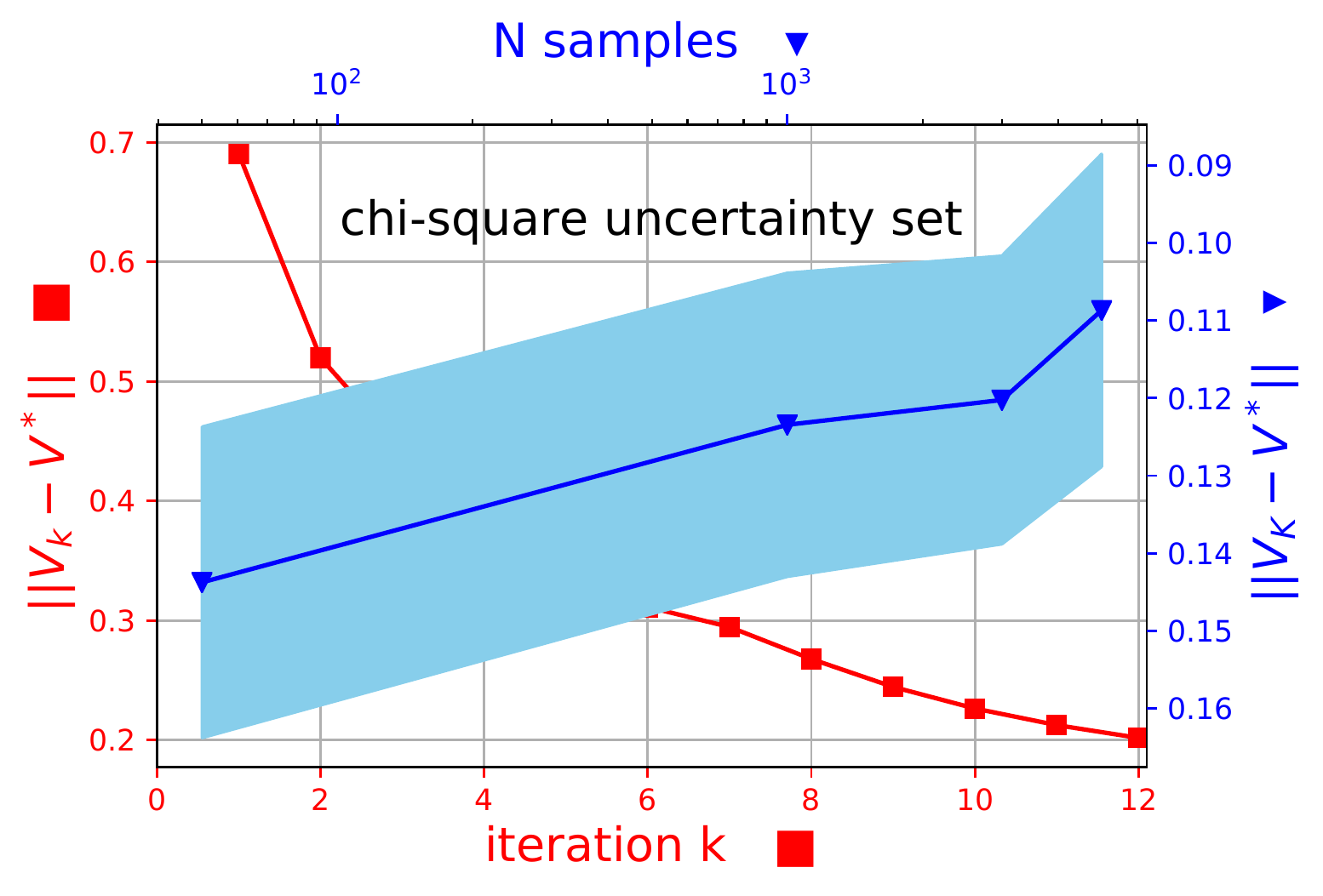}
	\end{minipage}
	\begin{minipage}{.24\textwidth}
		\centering
		\includegraphics[width=\linewidth]{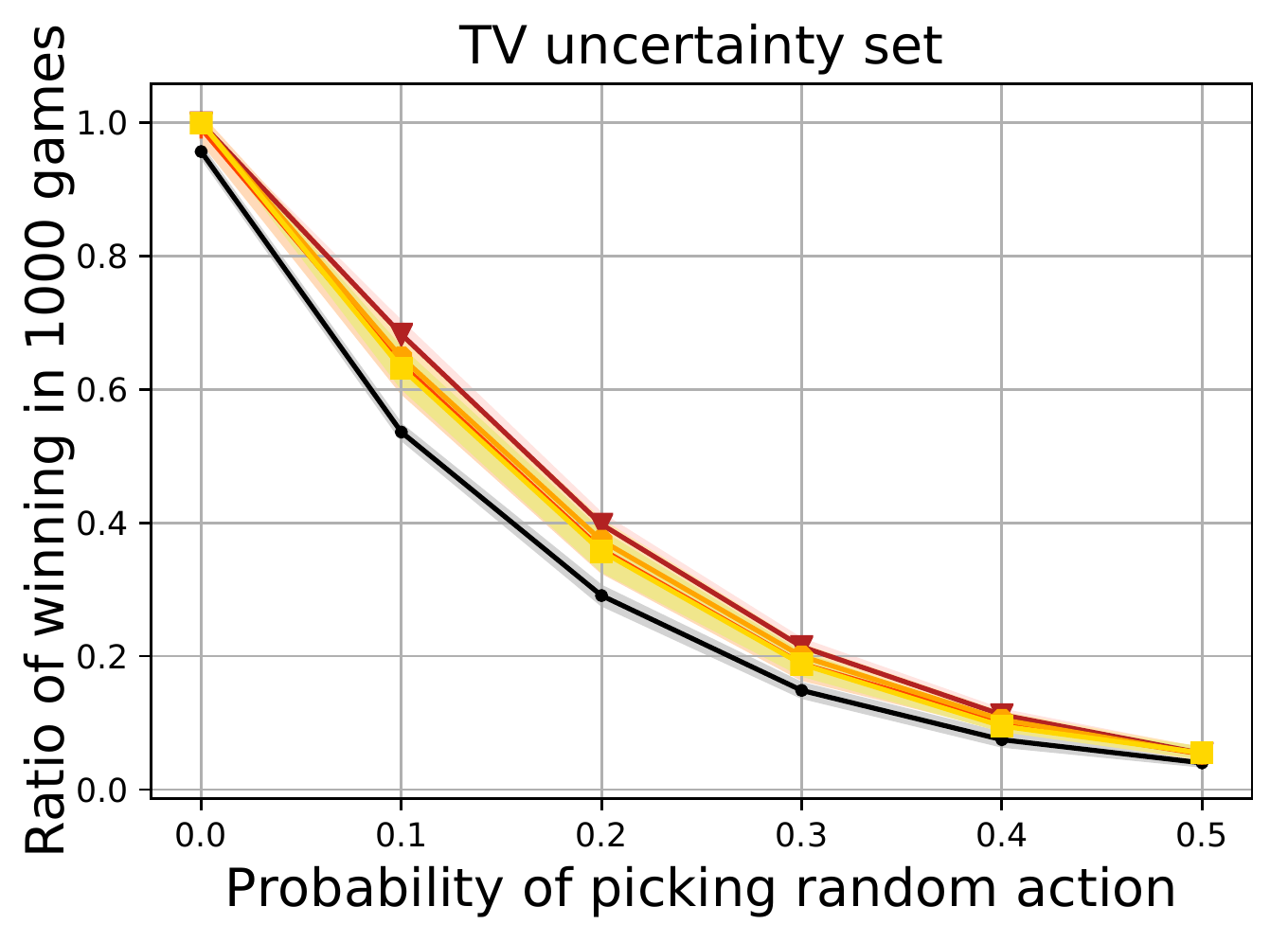}
	\end{minipage}
	\begin{minipage}{.24\textwidth}
		\centering
		\includegraphics[width=\linewidth]{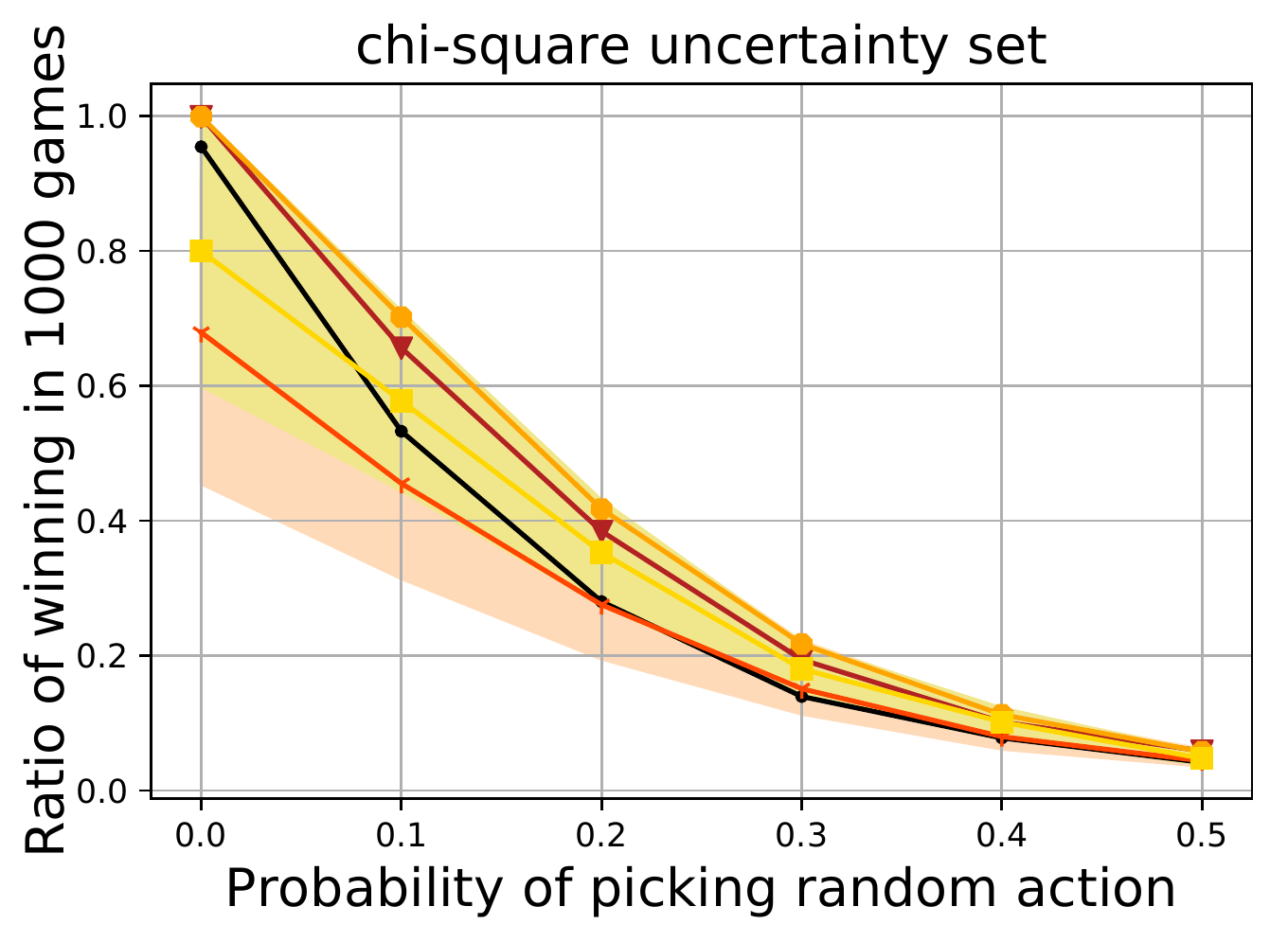}
	\end{minipage}
	\centering
	\captionof{figure}{\textit{Experiment results for the {FrozenLake8x8} environment.} The first two plots shows the rate of convergence with respect to the number of iterations ($k$) and the rate of convergence with respect to the number of samples $(N)$ for the TV and chi-square uncertainty set, respectively. The third and fourth plots shows the robustness of the learned policy against changes in the model parameter (probability of picking a random action).}
	\label{fig:frozenlake}
\end{figure*}

In this section we demonstrate the convergence behavior and robust performance of our REVI algorithm using numerical experiments. We consider two different settings, namely, the \emph{Gambler's Problem}  environment \citep[Example 4.3]{sutton2018reinforcement} and {\emph{FrozenLake8x8}} environment in OpenAI Gym \citep{brockman2016openai}. We also consider the TV uncertainty set and chi-square uncertainty set. We solve the optimization problem $\sigma_{\widehat{\Pp}}$ and $\sigma_{{\Pp}}$ 
using the Scipy \citep{scipy20} optimization library.

We illustrate  the following important characteristics of the REVI algorithm: 

(1) Rate of convergence with respect to the number of iterations: To demonstrate this, we plot $\norm{V_{k} - V^{*}}$ against the iteration number $k$, where $V_{k}$ is the value at the $k$th step of the REVI algorithm with $N=5000$. We compute $V^{*}$ using the full knowledge of the uncertainty set for benchmarking the performance of the REVI algorithm. \\
(2) Rate of convergence with respect to the number of samples: To show this, we plot $\norm{V_{K}(N) - V^{*}}$ against the number of samples $N$, where $V_{K}(N)$ is final value obtained from the REVI algorithm using $N$ samples. \\
(3) Robustness of the learned policy: To demonstrate this, we plot the number of times the robust policy $\pi_{K}$ (obtained from the REVI algorithm) successfully completed the task as a function of the change in an environment parameter. We perform $1000$ trials for each environment and each uncertainty set, and plot the fraction of the success. 



\textbf{Gambler's Problem:} In {gambler's problem}, a gambler starts with a random balance in her account and makes bets on a sequence of coin flips, winning her stake with heads and losing with tails, until she wins $\$100$ or loses all money. This problem can be formulated as a chain MDP with states in $\{1,\cdots,99\}$ and when in state $s$ the available actions are in $\{0,1,\cdots,\min(s,100-s)\}$. The agent is rewarded $1$ after reaching a goal and rewarded $0$ in every other timestep. The biased coin probability is fixed throughout the game. We denote its heads-up probability as $p_h$ and use $0.6$ as a nominal model for training our  algorithm. We also fix $c_{r} = 0.2$ for the chi-square uncertainty set experiments and $c_{r} = 0.4$ for the TV uncertainty set experiments. 

The red curves with square markers in the first two plots in  Fig. \ref{fig:gamblers} show the rate of convergence with respect to the number of iterations for TV and chi-square uncertainty sets respectively. As expected, convergence is fast due to the contraction property of the robust Bellman operator.


The blue curves with triangle markers in the first two plots in  Fig. \ref{fig:gamblers} show  the rate of convergence with respect to the number of samples for TV and chi-square uncertainty sets. We generated these curves for 10 different seed runs. The bold line depicts the mean of these runs and the error bar is the standard deviation. As expected, the plots show that $V_{K}(N)$ converges to $V^{*}$ as $N$ increases.




We then demonstrate the robustness of the approximate robust policy  $\pi_{K}$ (obtained with $N= 100, 500, 3000$) by evaluating its performance on environments with different values of $p_{h}$. We plot the fraction of the wins out of $1000$ trails. We also plot the performance the optimal robust policy $\pi^{*}$ as a benchmark. The third and fourth plot in Fig. \ref{fig:gamblers} show the results with TV and chi-square uncertainty sets respectively. We note that the performance of the non-robust policy decays drastically as we decrease the parameter $p_{h}$ from its nominal value $0.6$. On the other hand, the optimal robust policy performs consistently better under this change in the environment. We also note that $\pi_{K}(N)$ closely follows the performance of $\pi^{*}$ for large enough $N$.



\textbf{Frozen Lake environment:} {\emph{FrozenLake8x8}} is a gridworld environment of size $8\times 8$. It consists of some flimsy tiles which makes the agent fall into the water. The agent is rewarded $1$ after reaching a goal tile without falling and rewarded $0$ in every other timestep. We use the {\emph{FrozenLake8x8}} environment with default design as our nominal model except that we make the probability of transitioning to a state in the intended direction to be $0.4$ (the default value is $1/3$). We also set $c_{r} = 0.35$ for the chi-square uncertainty set experiments and $c_{r} = 0.7$ for the TV uncertainty set experiments. 

The red curves in the first two plots in  Fig. \ref{fig:frozenlake} show the rate of convergence with respect to the number of iterations for TV and chi-square uncertainty sets respectively. The blue curves in the first two plots in  Fig. \ref{fig:frozenlake} show  the rate of convergence with respect to the number of samples for TV and chi-square uncertainty sets respectively. The behavior is similar to the one  observed in the case of gambler's problem.



We demonstrate the robustness of the learned policy by evaluating it on FrozenLake test environments with  action perturbations. In the real-world settings, due to model mismatch or noise in the environments, the resulting action can be different from the intended action. We model this by picking a random action with some probability at each time step. In addition, we change the probability of transitioning to a state in the intended direction to be $0.2$ for these test environments. We observe that the performance of the robust RL policy is consistently better than the non-robust policy as we introduce model mismatch in terms of the probability of picking random actions. We also note that $\pi_{K}(N)$ closely follows the performance of $\pi^{*}$ for large enough $N$.



We note that we have  included our code for experiments in this \href{https://github.com/kishanpb/RobustRL}{GitHub page}. We note that we can employ a hyperparameter learning strategy to find the best value of $c_{r}$. We demonstrate this on the FrozenLake environment for the TV uncertainty set. We computed the optimal robust  policy for each $c_r \in \{0.1,0.2, \ldots, 1.6 \}$. We tested these policies across $1000$ games with  random action probabilities $\{0.1,0.2, \ldots, 0.5\}$ on the test environment. We found that the policy for $c_r=1.2$ has the best winning ratio across all the random action probabilities. We do not exhibit exhaustive experiments on this hyperparameter learning strategy as it is out of scope of the intent of this manuscript.
\section{Conclusion and Future Work}


We presented a model-based robust reinforcement learning algorithm called Robust Empirical Value Iteration algorithm, where we used an approximate robust Bellman updates in the vanilla robust value iteration algorithm.  We provided a finite sample performance characterization of the learned policy with respect to the optimal robust policy for three different uncertainty sets, namely, the total variation uncertainty set, the chi-square uncertainty set, and the Kullback-Leibler uncertainty set. We also demonstrated the performance of REVI algorithm on two different environments showcasing its theoretical properties of convergence. We also showcased the REVI algorithm based policy being robust to the changes in the environment as opposed to the non-robust policies.

The goal of this work was to develop the fundamental theoretical results for the finite state space and action space regime. As mentioned earlier, the sub-optimality of the sample complexity of our REVI algorithm in factors $|\Ss|$ and $1/(1-\gamma)$ needs more investigation and refinements in the analyses. In the future, we will extend this idea to robust RL with linear and nonlinear function approximation architectures and for more general models in deep RL. 
\section*{Acknowledgments}

Dileep Kalathil gratefully acknowledges funding from the U.S. National Science Foundation (NSF) grants NSF-EAGER-1839616, NSF-CRII-CPS-1850206 and NSF-CAREER-EPCN-2045783.

\bibliography{References-RobustRL}

\appendix
\onecolumn
\section*{Appendix}

\section{Useful Technical Results}
In this section we state some  existing results  that are useful in our analysis.

\begin{lemma}[Hoeffding's inequality \hspace{-0.1cm} \text{\cite[Theorem 2.2.6]{vershynin2018high}}] \label{lem:hoeffding_ineq}
Let $X_1,\cdots, X_T$ be independent random variables. Assume that $X_t \in [m_t, M_t]$ for every $t$ with $M_t > m_t$. Then, for any $\epsilon > 0$, we have \[ \pr \left( \sum_{t=1}^T (X_t - \E[X_t]) \geq \epsilon \right) \leq \exp\left( -\frac{2\epsilon^2}{\sum_{t=1}^T (M_t - m_t)^2} \right). \]
\end{lemma}

\begin{lemma}[Self-bounding variance inequality  \text{\cite[Theorem 10]{MaurerP09}}] \label{lem:variance_ineq}
Let $X_1,\cdots, X_T$ be independent and identically distributed random variables with finite variance, that is, $\mathrm{Var}(X_1) < \infty$. Assume that $X_t \in [0, M]$ for every $t$ with $M > 0$, and let $S_T^2 = \frac{1}{T} \sum_{t=1}^T X_t^2 - (\frac{1}{T} \sum_{t=1}^T X_t)^2.$ Then, for any $\epsilon > 0$, we have \[ \pr \left( |S_T - \sqrt{\mathrm{Var}(X_1)} | \geq \epsilon \right) \leq 2 \exp\left( -\frac{T \epsilon^2}{2M^2} \right). \]
\end{lemma}
\begin{proof}
    The proof of this lemma directly follows from \cite[Theorem 10]{MaurerP09} by noting that we can rewrite $S_T^2$ as follows \[ \frac{T}{T-1}S_T^2 = \frac{1}{T(T-1)} \sum_{i,j=1}^T (X_i-X_j)^2.\] Also, note that we apply \cite[Theorem 10]{MaurerP09} for the scaled random variables $X_t/M\in[0,1]$.
\end{proof}

We now provide a covering number result that is useful to get high probability concentration bounds for value function classes. We first define  minimal $\eta$-cover of a set.

\begin{definition}[Minimal $\eta$-cover; \text{\citep[Definition 5.5]{van2014probability}}]
A set $\N_{\V}(\eta)$ is called an $\eta$-cover for a metric space $(\V, d)$ if for every $V\in\V$, there exists a $V'\in\N$ such that $d(V,V')\leq \eta$. Furthermore, $\N_{\V}(\eta)$ with the minimal cardinality ($|\N_{\V}(\eta)|$) is called a minimal $\eta$-cover.
\end{definition}

From \citep[Exercise 5.5 and Lemma 5.13]{van2014probability} we have the following result.
\begin{lemma}[Covering Number]
\label{por:covering_num}
Let $ \V = \{V\in \R^{|\Ss|}~:~ \|V\| \leq V_{\max}\}$. Let $\N_\V(\eta)$ be a minimal $\eta$-cover of $\V$ with respect to the distance metric $d(V,V') = \| V-V'\|$ for some fixed $\eta\in (0,1).$ Then we have 
	$$
	\log |\N_\V(\eta)|  \leq |\Ss| \cdot \log \left( \frac{3\,V_{\max}}{\eta}  \right) .	$$
\end{lemma}
\begin{proof}
    We will consider the normalized metric space $(\V_n, d_n)$,
    where \[ \V_n := \V/V_{\max} = \{V\in \R^{|\Ss|}~:~ \|V\| \leq 1\} \] and $d_n := d/V_{\max}$ to make use of the fact that  the covering number is invariant to the scaling of a metric space. Let $\eta_n := \eta/V_{\max}$. Then, it follows from \citep[Exercise 5.5 and Lemma 5.13]{van2014probability} that \[ \log |\N_\V(\eta)| = \log |\N_{\V_n}(\eta_n)| \leq |\Ss| \cdot \log \left( \frac{3}{\eta_n}  \right)  = |\Ss| \cdot \log \left( \frac{3\,V_{\max}}{\eta}  \right) .	 \] This completes the proof.
\end{proof}

Here we present another covering number result, with a similar proof as Lemma \ref{por:covering_num}, that is useful to get our upperbound for the KL uncertainty set.
\begin{lemma}[Covering Number of a bounded real line]
\label{por:covering_num_real_line}
Let $\Theta \subset \mathbb{R}$ with  $\Theta = [l, u]$ for some real numbers $u>l$. Let $\N_{\Theta}(\eta)$ be a minimal $\eta$-cover of $\Theta$ with respect to the distance metric $d(\theta,\theta') = |\theta-\theta'|$ for some fixed $\eta\in (0,1).$ Then we have 
	$
	 |\N_\Theta(\eta)|  \leq  {3(u-l)}/{\eta}   .	$
\end{lemma}

\section{Proof of the Theorems}

\subsection{Concentration Results}
Here, we prove Lemma \ref{lem:covering-hoeffdings}. We state the following result first. 
\begin{lemma}
\label{lem:hoeffding_concentration}
For any  $V\in \R^{|\Ss|}$ with $\|V\| \leq V_{\max}$,  with probability at least $1-\delta$,
\begin{align*}
\max_{(s,a)} |{P^o_{s,a}} V -  \widehat{P}_{s,a} V | \leq V_{\max} \sqrt{\frac{\log(2|\Ss||\Aa|/\delta)}{2N}}
\end{align*}
\end{lemma}
\begin{proof}
Fix any $(s,a)$ pair. 	Consider a discrete random variable $X$ taking value $V(j)$ with probability $P^o_{s,a}(j)$ for all $j\in\{1,2,\cdots,|\Ss|\}$. From Hoeffding's inequality  (Lemma \ref{lem:hoeffding_ineq}), we have \begin{align*} 
&\pr (P^o_{s,a} V -  \widehat{P}_{s,a} V  \geq \epsilon ) \leq \exp(-2  N \epsilon^2/V^2_{\max}),\quad \pr (\widehat{P}_{s,a} V -  P^o_{s,a} V  \geq \epsilon ) \leq \exp(-2  N \epsilon^2/V^2_{\max}).
\end{align*}
Choosing $\epsilon = V_{\max} \sqrt{\frac{\log(2|\Ss||\Aa|/\delta)}{2N}},$ we get $ \pr (|P^o_{s,a} V -  \widehat{P}_{s,a} V | \geq V_{\max} \sqrt{\frac{\log(2|\Ss||\Aa|/\delta)}{2N}} ) \leq \frac{\delta}{|\Ss||\Aa|}$.
 Now, using union bound, we get 
 \begin{align*}
	&\pr ( \max_{(s,a)} |P^o_{s,a} V -  \widehat{P}_{s,a} V | \geq V_{\max} \sqrt{\frac{\log(2|\Ss||\Aa|/\delta)}{2N}} ) \leq \sum_{s,a} \pr (|P^o_{s,a} V -  \widehat{P}_{s,a} V | \geq V_{\max} \sqrt{\frac{\log(2|\Ss||\Aa|/\delta)}{2N}} ) \hspace{-0.05cm}\leq\hspace{-0.05cm} \delta. 
	\end{align*} 
This completes the proof.
\end{proof}

\begin{proof}[\textbf{Proof of Lemma \ref{lem:covering-hoeffdings}:}]
   Let  $\mathcal{V} = \{V \in \mathbb{R}^{|\Ss|}: \norm{V}_{\infty} \leq 1/(1 - \gamma)\}$.  Let $\N_\V(\eta)$ be a minimal $\eta$-cover of $\mathcal{V}$. Fix a $\mu \leq V$.  By the definition of $\N_\V(\eta)$, there exists a $\mu' \in \N_\V(\eta)$ such that $\norm{\mu - \mu'} \leq \eta$. Now, for these particular $\mu$ and $\mu'$, we get 
\begin{align*}
| \widehat{P}_{s,a} \mu - P^o_{s,a} \mu| &\leq    | \widehat{P}_{s,a} \mu- \widehat{P}_{s,a} \mu'|  +  | \widehat{P}_{s,a} \mu' - P^o_{s,a} \mu'|  + | P^o_{s,a} \mu'- P^o_{s,a} \mu| \\
&\stackrel{(a)}{\leq}  \| \widehat{P}_{s,a}\|_1 \|\mu-\mu'\|_{\infty} +     | \widehat{P}_{s,a} \mu' - P^o_{s,a} \mu'| +  \| P^o_{s,a}\|_1 \|\mu' - \mu\|_{\infty} \\
&\leq \sup_{\mu'\in \N_\V(\eta)} \max_{s,a}| \widehat{P}_{s,a} \mu' - P^o_{s,a} \mu'| + 2\eta
\end{align*} where $(a)$ follows from H$\ddot{\text{o}}$lder's inequality. Now, taking max on both sides with respect to $\mu$ and $(s,a)$ we get
\begin{align*}
\sup_{\mu\in \V} \max_{s,a} | \widehat{P}_{s,a} \mu - P^o_{s,a} \mu| &\leq \sup_{\mu'\in \N_\V(\eta)} \max_{s,a}| \widehat{P}_{s,a} \mu' - P^o_{s,a} \mu'| + 2\eta \\
&\stackrel{(b)}{\leq} \frac{1}{1-\gamma} \sqrt{\frac{\log(4|\Ss||\Aa| |\N_\V(\eta) |/\delta)}{2N}} + 2\eta \\
	    &\stackrel{(c)}{\leq} \frac{1}{1-\gamma} \sqrt{\frac{|\Ss|\log(12|\Ss||\Aa|/(\delta\eta(1-\gamma)))}{2N}} + 2\eta,
	\end{align*} 
with probability at least $1-\delta/2$. Here,  $(b)$ follows from Lemma \ref{lem:hoeffding_concentration} and the union bound and $(c)$ from Lemma \ref{por:covering_num}. 
\end{proof}

\subsection{Proof of Theorem \ref{thm:revi-TV-guarantee} }


\begin{proof}[\textbf{Proof of Lemma \ref{lem:sigma_v_diff}}]
	We only prove the first inequality since the proof is analogous for the other inequality. For any $(s,a)\in\ScA$ we have  
	\begin{align}
	 \sigma_{{{\Pp}_{s,a}}} (V_2) - \sigma_{{{\Pp}_{s,a}}} (V_1) &=  \inf_{q \in {{\Pp}_{s,a}}} q^\top V_2 - \inf_{\tilde{q} \in {{\Pp}_{s,a}}} \tilde{q}^\top V_1 = \inf_{q \in {{\Pp}_{s,a}} } \sup_{\tilde{q} \in {{\Pp}_{s,a}}} q^\top V_2 - \tilde{q}^\top V_1 \nonumber  \\
	&\geq \inf_{q \in {{\Pp}_{s,a}} } q^\top (V_2 - V_1) = \sigma_{{\Pp}_{s,a}} (V_2 - V_1).
	\label{eq:lemma_sigma_d-1}
	\end{align} By definition, for any arbitrary $\epsilon > 0$, there exists a $P_{s,a}  \in \Pp_{s,a} $ such that
	\begin{align}
	\label{eq:lemma_sigma_d-2}
	P^{\top}_{s,a} (V_{2} - V_{1}) - \epsilon \leq \sigma_{{\Pp_{s,a}}} (V_2 - V_1).
	\end{align}
	Using \eqref{eq:lemma_sigma_d-2} and \eqref{eq:lemma_sigma_d-1},
	\begin{align*}
	 \sigma_{{{\Pp}_{s,a}}} (V_1) - \sigma_{{{\Pp}_{s,a}}} (V_2)  &\leq  P^{\top}_{s,a}  (V_{1} - V_{2}) +  \epsilon   \stackrel{(a)}{\leq}    \|P_{s,a}\|_1  \|V_{1} - V_{2}\|  +  \epsilon
	=    \|V_{1} - V_{2}\|  +  \epsilon
	\end{align*}
	where $(a)$ follows from Holder's inequality. Since $\epsilon$ is arbitrary, we get, $\sigma_{{{\Pp}_{s,a}}} (V_1) - \sigma_{{{\Pp}_{s,a}}} (V_2)  \leq   \|V_{1} - V_{2}\| $. Exchanging the roles of $V_1$ and $V_2$ completes the proof.
\end{proof}

\begin{proof}[\textbf{Proof of Lemma \ref{lem:tv-sigma-diff}}]
Fix any $(s,a)$ pair. From \cite[Lemma 4.3]{iyengar2005robust} we have that \begin{align}
\label{eq:sigma-TV}
&\sigma_{{\Pp}^{\mathrm{tv}}_{s,a}} (V) = P^o_{s,a} V +  \max_{\mu: 0 \leq \mu \leq V} \bigg( -P^o_{s,a} \mu - c_r\max_{s} (V_\mu(s)) + c_r\min_{s} (V_\mu(s)) \bigg) \\
&\sigma_{\widehat{\Pp}^{\mathrm{tv}}_{s,a}} (V) = \widehat{P}_{s,a} V + \max_{\mu: 0 \leq \mu \leq V} \bigg( -\widehat{P}_{s,a} \mu - c_r\max_{s} (V_\mu(s)) + c_r\min_{s} (V_\mu(s)) \bigg),
\end{align} 
where $0\leq\mu\leq V$ is an elementwise inequality and $V_\mu(s)=V(s)-\mu(s)$ for all $s\in\Ss$.
	
Using the fact that $|\max_{x} f(x)-\max_{x} g(x)| \leq \max_{x} |f(x) - g(x)|$, it directly follows that
\[ \begin{split} &| \sigma_{\widehat{\Pp}^{\mathrm{tv}}_{s,a}} (V) - \sigma_{{\Pp}^{\mathrm{tv}}_{s,a}} (V)  | \leq | \widehat{P}_{s,a} V - P^o_{s,a} V| + \max_{\mu: 0 \leq \mu \leq V} | \widehat{P}_{s,a} \mu - P^o_{s,a} \mu|. \end{split} \] 
Further simplifying we get \begin{align*}
| \sigma_{\widehat{\Pp}^{\mathrm{tv}}_{s,a}} (V) - \sigma_{{\Pp}^{\mathrm{tv}}_{s,a}} (V)  | &\leq | \widehat{P}_{s,a} V - P^o_{s,a} V| + \max_{\mu: 0 \leq \mu \leq V} | \widehat{P}_{s,a} \mu - P^o_{s,a} \mu| \\
&\leq \max_{\mu \in \V} | \widehat{P}_{s,a} \mu - P^o_{s,a} \mu| + \max_{\mu: 0 \leq \mu \leq V} | \widehat{P}_{s,a} \mu - P^o_{s,a} \mu| \leq  2  \max_{\mu \in \V} | \widehat{P}_{s,a} \mu - P^o_{s,a} \mu|.
\end{align*}
This completes the proof.
\end{proof}

We are now ready to prove Proposition \ref{lem:tv-sigma-diff-bound-uniform}.

\begin{proof}[\textbf{Proof of  Proposition \ref{lem:tv-sigma-diff-bound-uniform}}]
For any given $V \in \V$ and $(s,a)$, from Lemma \ref{lem:tv-sigma-diff},  we have
\begin{align*}
&| \sigma_{\widehat{\Pp}^{\mathrm{tv}}_{s,a}} (V) - \sigma_{{\Pp}^{\mathrm{tv}}_{s,a}} (V)  | \leq  2  \max_{\mu \in \V} | \widehat{P}_{s,a} \mu - P^o_{s,a} \mu| \leq  2  \max_{\mu \in \V} \max_{s,a} | \widehat{P}_{s,a} \mu - P^o_{s,a} \mu|.
\end{align*}
Taking the maximum over $V$ and  $(s, a)$ on both sides, we get
\begin{align}
\label{eq:prop3-pf-step1}
\max_{V \in \V} ~ \max_{s,a}  ~| \sigma_{\widehat{\Pp}^{\mathrm{tv}}_{s,a}} (V) - \sigma_{{\Pp}^{\mathrm{tv}}_{s,a}} (V) |  \leq 2  \max_{\mu \in \V} \max_{s,a} | \widehat{P}_{s,a} \mu - P^o_{s,a} \mu|.
\end{align}
Now, from the proof of Lemma \ref{lem:covering-hoeffdings},  for any $\eta, \delta \in (0, 1)$, we get
\begin{align}
\label{eq:prop3-pf-step2}
\max_{\mu\in \V} \max_{s,a} | \widehat{P}_{s,a} \mu - P^o_{s,a} \mu| &\leq \frac{1}{1-\gamma} \sqrt{\frac{|\Ss|\log(6|\Ss||\Aa|/(\delta\eta(1-\gamma)))}{2N}} + 2\eta,
\end{align} 
with probability greater than $1-\delta$. Using \eqref{eq:prop3-pf-step2} in \eqref{eq:prop3-pf-step1}, we get the desired result. 
\end{proof}

We also need the following result that specifies the amplification when replacing the algorithm iterate value function with the value function of the policy towards approximating the optimal value. 
\begin{lemma}
\label{lem:singh_yee_main_result}
Let $V_{k}$ and $Q_{k}$ be as given in the REVI algorithm for $k \geq 1$. Also, let $\pi_{k} = \argmax_{a} Q_{k}(s, a)$. Then, 
\begin{align*}
\| \widehat{V}^{*} - \widehat{V}^{\pi_k} \| \leq \frac{2\gamma}{1-\gamma} \| V_k  - \widehat{V}^{*} \|.
\end{align*}
Furthermore, \begin{align*}
    \| {V}^{*} - {V}^{\pi_k} \| \leq \frac{2}{1-\gamma} \| Q_k  - {Q}^{*} \|.
\end{align*}
\end{lemma}
\begin{proof}
The proof is similar to the proof in \citep[Main Theorem, Corollary 2]{singh1994upper}. A straight forward modification to this proof, using the fact that ${\sigma}_{\widehat{\Pp}_{s,a}}$ and ${\sigma}_{{\Pp}_{s,a}}$  are  $1$-Lipschitz functions   as shown in Lemma \ref{lem:sigma_v_diff}, will give the desired result. 
\end{proof}

\begin{proof}[\textbf{Proof of Theorem \ref{thm:revi-TV-guarantee}}]

Recall  the empirical RMDP  $\widehat{M} = (\Ss, \Aa, r, \widehat{\Pp}^{\mathrm{tv}}, \gamma)$.  For any policy $\pi$, let $\widehat{V}^{\pi}$ be robust value function of policy $\pi$ with respect to the RMDP $\widehat{M}$. The optimal robust policy,  value function, and state-action value function of $\widehat{M}$ are denoted as $\widehat{\pi}^{\star}, \widehat{V}^{\star}$ and $ \widehat{Q}^{\star}$, respectively. Also, for any policy $\pi$, we have  $\widehat{Q}^{\pi}(s,a) = r(s,a) + \gamma \sigma_{\widehat{\Pp}^{\mathrm{tv}}_{s,a}}(\widehat{V}^{\pi})$ and ${Q}^{\pi}(s,a) = r(s,a) + \gamma \sigma_{{\Pp}^{\mathrm{tv}}_{s,a}}({V}^{\pi})$. 

Let $V_{k}$ and $Q_{k}$ be as given in the REVI algorithm for $k \geq 1$. Also, let $\pi_{k}(s) = \argmax_{a} Q_{k}(s, a)$.  Now, 
 \begin{align}
 \label{eq:traingle-split}
     \| V^* - V^{\pi_k} \| \leq \| V^* - \widehat{V}^{*} \| + \| \widehat{V}^{*} - \widehat{V}^{\pi_{k}} \| + \| \widehat{V}^{\pi_{k}} - V^{\pi_{k}} \|.
 \end{align}

 \textit{ 1) Bounding the first term in \eqref{eq:traingle-split}:} 
Let  $\mathcal{V} = \{V \in \mathbb{R}^{|\Ss|}: \norm{V} \leq 1/(1 - \gamma)\}$.  For any $s\in\Ss$, 
\begin{align*}
V^*(s) -  \widehat{V}^{*}(s) &= Q^*(s,\pi^*(s)) -  \widehat{Q}^*(s,\hat{\pi}^*(s)) 	\stackrel{(a)}{\leq}  Q^*(s,\pi^*(s)) -  \widehat{Q}^*(s,\pi^*(s)) \\	
&\stackrel{(b)}{=} \gamma \sigma_{{\Pp}^{\mathrm{tv}}_{s,\pi^*(s)}} (V^*) - \gamma \sigma_{\widehat{\Pp}^{\mathrm{tv}}_{s,\pi^*(s)}} (\widehat{V}^*)\\
&= \gamma  (\sigma_{{\Pp}^{\mathrm{tv}}_{s,\pi^*(s)}} (V^*) -  \sigma_{\widehat{\Pp}^{\mathrm{tv}}_{s,\pi^*(s)}} ({V}^*) ) + \gamma  (\sigma_{\widehat{\Pp}^{\mathrm{tv}}_{s,\pi^*(s)}} ({V}^*) - \sigma_{\widehat{\Pp}^{\mathrm{tv}}_{s,\pi^*(s)}} (\widehat{V}^*) )\\
&\stackrel{(c)}{\leq}  \gamma  (\sigma_{{\Pp}^{\mathrm{tv}}_{s,\pi^*(s)}} (V^*) -  \sigma_{\widehat{\Pp}^{\mathrm{tv}}_{s,\pi^*(s)}} ({V}^*) )  + \gamma \|V^*-\widehat{V}^{*}\| \\
&\leq \gamma  ~ \max_{V \in \V} ~ \max_{s,a}  ~| \sigma_{\widehat{\Pp}^{\mathrm{tv}}_{s,a}} (V) - \sigma_{{\Pp}^{\mathrm{tv}}_{s,a}} (V) | + \gamma \|V^*-\widehat{V}^{*}\| 
\end{align*} 
where $(a)$ follows since ${\hat{\pi}^*}$ is the robust optimal policy for $\widehat{M}$, $(b)$ follows from the definitions of $Q^*$ and $\widehat{Q}^*$,  $(c)$ follows from Lemma \ref{lem:sigma_v_diff}. Similarly analyzing for $  \widehat{V}^{*}(s) - V^*(s)$, we get
\begin{align}
\|V^*-\widehat{V}^{*}\| \leq \frac{\gamma}{(1-\gamma)} ~ \max_{V \in \V} ~ \max_{s,a}  ~| \sigma_{\widehat{\Pp}^{\mathrm{tv}}_{s,a}} (V) - \sigma_{{\Pp}^{\mathrm{tv}}_{s,a}} (V) |. \label{eq:tv-term1-incomplete-bound}
\end{align} 
Now,   using  Proposition \ref{lem:tv-sigma-diff-bound-uniform}, with  probability greater than $1 - \delta$, we get 
\begin{align}
\label{eq:tv-term1-bound}
\|V^*-\widehat{V}^{*}\| \leq \frac{\gamma}{(1-\gamma)} {C}^{\mathrm{tv}}_{u}(N,\eta, \delta),
\end{align}
where 	${C}^{\mathrm{tv}}_{u}(N,\eta, \delta)$ is given in  equation  \eqref{eq:c-tv-uniform} in the statement of Proposition \ref{lem:tv-sigma-diff-bound-uniform}.

\textit{2) Bounding the second term in \eqref{eq:traingle-split}:}  Let  $\widehat{T}$ be the robust Bellman operator corresponding to $\widehat{M}$. So,  $\widehat{T}$ is a $\gamma$-contraction mapping and $\widehat{V}^*$ is its unique fixed point \citep{iyengar2005robust}.  The REVI iterates $V_{k}, k \geq 0$, with $V_{0} = 0$, can now be expressed as $V_{k+1} = \widehat{T} V_{k}$. Using the  properties of  $\widehat{T}$, we get
\begin{align}
\|V_k -  \widehat{V}^*\| = \|\widehat{T} V_{k-1} -  \widehat{T} \widehat{V}^*\|	\leq \gamma \|V_{k-1} -  \widehat{V}^*\| \leq \cdots \leq \gamma^{k} \|V_{0} -  \widehat{V}^*\| \leq {\gamma^k}/{(1-\gamma)}.
\label{eq:vk-diff-v-star}
\end{align} 
Now, using Lemma \ref{lem:singh_yee_main_result},  we get
 \begin{align}
 \label{eq:tv-term2-bound}
 \|\widehat{V}^{\pi_{k}} -  \widehat{V}^*\|  \leq \frac{2\gamma^{k+1}}{(1-\gamma)^2}.
 \end{align}

\textit{3) Bounding the third term in \eqref{eq:traingle-split}:}  This is similar to bounding the first term. For any $s\in\Ss$, 
 \begin{align*}
 V^{\pi_k}(s) -  \widehat{V}^{\pi_k}(s) &=   Q^{\pi_k}(s,\pi_k(s)) -  \widehat{Q}^{\pi_k}(s,\pi_k(s)) =  \gamma \sigma_{{\Pp}_{s,\pi_k(s)}} (V^{\pi_k}) - \gamma \sigma_{\widehat{\Pp}_{s,\pi_k(s)}} (\widehat{V}^{\pi_k})\\
 &= \gamma  (\sigma_{{\Pp}_{s,\pi_k(s)}} (V^{\pi_k}) -  \sigma_{{\Pp}_{s,\pi_k(s)}} (\widehat{V}^{\pi_k}))  + \gamma ( \sigma_{{\Pp}_{s,\pi_k(s)}} (\widehat{V}^{\pi_k}) - \sigma_{\widehat{\Pp}_{s,\pi_k(s)}} (\widehat{V}^{\pi_k}) )\\
 &\stackrel{(d)}{\leq}  \gamma \|V^{\pi_k}-\widehat{V}^{\pi_k}\| + \gamma ( \sigma_{{\Pp}_{s,\pi_k(s)}} (\widehat{V}^{\pi_k}) - \sigma_{\widehat{\Pp}_{s,\pi_k(s)}} (\widehat{V}^{\pi_k}) )\\
  &\leq  \gamma \|V^{\pi_k}-\widehat{V}^{\pi_k}\| + \gamma  ~ \max_{V \in \V} ~ \max_{s,a}  ~| \sigma_{\widehat{\Pp}^{\mathrm{tv}}_{s,a}} (V) - \sigma_{{\Pp}^{\mathrm{tv}}_{s,a}} (V) | 
 \end{align*} 
 where $(d)$ follows  from Lemma \ref{lem:sigma_v_diff}. Similarly analyzing for $  \widehat{V}^{\pi_k}(s) - V^{\pi_k}(s)$, we get,
\begin{align}
\|V^{\pi_k}-\widehat{V}^{\pi_k}\| \leq \frac{\gamma}{(1-\gamma)} ~ \max_{V \in \V} ~ \max_{s,a}  ~| \sigma_{\widehat{\Pp}^{\mathrm{tv}}_{s,a}} (V) - \sigma_{{\Pp}^{\mathrm{tv}}_{s,a}} (V) | . \label{eq:tv-term3-incomplete-bound}
\end{align}
 Now,   using  Proposition \ref{lem:tv-sigma-diff-bound-uniform}, with  probability greater than $1 - \delta$, we get 
\begin{align}
\label{eq:tv-term3-bound}
\|V^{\pi_k}-\widehat{V}^{\pi_k}\| \leq \frac{\gamma}{(1-\gamma)} {C}^{\mathrm{tv}}_{u}(N,\eta, \delta).
\end{align}
Using \eqref{eq:tv-term1-bound} -  \eqref{eq:tv-term3-bound} in \eqref{eq:traingle-split}, we get, with probability at least $1-2\delta$,
 \begin{align}
 \label{eq:tv-thm-pf-st21}
 \| V^* - V^{\pi_k} \| \leq \frac{2\gamma^{k+1}}{(1-\gamma)^2} +  \frac{2 \gamma}{(1-\gamma)} {C}^{\mathrm{tv}}_{u}(N,\eta, \delta).
 \end{align}
Using the value of ${C}^{\mathrm{tv}}_{u}(N,\eta, \delta)$ as given in  Proposition \ref{lem:tv-sigma-diff-bound-uniform}, we get
 \begin{align}
 \label{eq:tv-thm-bound-1}
 \| V^* - V^{\pi_k} \| \leq &\frac{2\gamma^{k+1}}{(1-\gamma)^2} +   \frac{4\gamma}{(1-\gamma)^2}  \sqrt{\frac{|\Ss|\log(6|\Ss||\Aa|/(\delta\eta(1-\gamma)))}{2N}} + \frac{8 \gamma \eta}{(1-\gamma)}  
 \end{align}
 with probability at least $1-2\delta$.

Now, choose $\eta = \epsilon (1 - \gamma) /(24 \gamma)$. Since $\epsilon \in (0, 24 \gamma/(1-\gamma))$, this particular $\eta$ is in $(0, 1)$. Now,  choosing
\begin{align}
\label{eq:K0-modified}
k &\geq K_{0} = \frac{1}{ \log(1/\gamma)}  \log(\frac{6\gamma}{\epsilon (1-\gamma)^2}), \\
\label{eq:N-TV-modified}
N &\geq N^{\mathrm{tv}} = 
	\frac{ 72 \gamma^2}{(1-\gamma)^4}  \frac{|\Ss|\log(144\gamma|\Ss||\Aa|/(\delta\epsilon(1-\gamma)^2))}{\epsilon^2},
\end{align}
we get $\| V^* - V^{\pi_k} \| \leq \epsilon$ with probability at least $1-2\delta$. 
 \end{proof}

\subsection{Proof of Theorem \ref{thm:revi-Chi-guarantee} }


\begin{proof}[\textbf{Proof of  Lemma \ref{lem:chi-sigma-diff}}]
Fix an $(s, a)$ pair.  From \cite[Lemma 4.2]{iyengar2005robust}, we have 
\begin{align}
\label{eq:sigma-chi}
&\sigma_{{\Pp}^{\mathrm{c}}_{s,a}} (V) = \max_{\mu: 0 \leq \mu \leq V} \bigg( P^{o}_{s,a} (V-\mu)  - \sqrt{c_r \mathrm{Var}_{P^{o}_{s,a}}(V-\mu)} \bigg), 
\end{align}
where $\mathrm{Var}_{P^{o}_{s,a}}(V-\mu) = P^{o}_{s,a}(V-\mu)^2 - (P^{o}_{s,a}(V-\mu))^2$. We get a similar expression for $\sigma_{\widehat{{\Pp}}^{\mathrm{c}}_{s,a}} (V) $.  Using these expressions,   with the additional facts that $|\max_{x} f(x)-\max_{x} g(x)| \leq \max_{x} |f(x) - g(x)|$ and $\max_{x} (f(x)+g(x)) \leq \max_{x} f(x)+\max_{x} g(x)$, we get  the desired result. 
\end{proof}

We state the following concentration result that is useful for the proof of Proposition \ref{lem:chi-sigma-diff-bound-uniform}. 
\begin{lemma}
\label{lem:variance_concentration}
For any  $V\in \R^{|\Ss|}_{+}$ with $\|V\| \leq V_{\max}$,  with probability at least $1-\delta$,
\begin{align*}
\max_{(s,a)} |\sqrt{\text{Var}_{P^o_{s,a}} V } -  \sqrt{\text{Var}_{\widehat{P}_{s,a}} V } | \leq V_{\max} \sqrt{\frac{2\log(2|\Ss||\Aa|/\delta)}{N}}
\end{align*}
\end{lemma}
\begin{proof}
Fix any $(s,a)$ pair. 	Consider a discrete random variable $X$ taking value $V(j)$ with probability $P^o_{s,a}(j)$ for all $j\in\{1,2,\cdots,|\Ss|\}$. From the Self-bounding variance inequality  (Lemma \ref{lem:variance_ineq}), we have \begin{align*} 
& \pr ( |\sqrt{\text{Var}_{P^o_{s,a}} V } -  \sqrt{\text{Var}_{\widehat{P}_{s,a}} V } | \geq \epsilon ) \leq 2 \exp(-  N \epsilon^2/(2 V^2_{\max})).
\end{align*}
Choosing $\epsilon = V_{\max} \sqrt{\frac{2\log(2|\Ss||\Aa|/\delta)}{N}},$ we get $ \pr (|P^o_{s,a} V -  \widehat{P}_{s,a} V | \geq V_{\max} \sqrt{\frac{2\log(2|\Ss||\Aa|/\delta)}{N}} ) \leq \frac{\delta}{|\Ss||\Aa|}$.
 Now, using union bound, we get 
 \begin{align*}
	&\pr ( \max_{(s,a)} |\sqrt{\text{Var}_{P^o_{s,a}} V } -  \sqrt{\text{Var}_{\widehat{P}_{s,a}} V } | \geq V_{\max} \sqrt{\frac{2\log(2|\Ss||\Aa|/\delta)}{N}} )
	\\&\hspace{1cm}\leq \sum_{s,a} \pr (|\sqrt{\text{Var}_{P^o_{s,a}} V } -  \sqrt{\text{Var}_{\widehat{P}_{s,a}} V } | \geq V_{\max} \sqrt{\frac{2\log(2|\Ss||\Aa|/\delta)}{N}} ) \hspace{-0.05cm}\leq\hspace{-0.05cm} \delta. 
	\end{align*} 
This completes the proof.
\end{proof}

We are now ready to prove Proposition \ref{lem:chi-sigma-diff-bound-uniform}.


\begin{proof}[\textbf{Proof of  Proposition \ref{lem:chi-sigma-diff-bound-uniform}}]
Fix an $(s,a)$ pair. From Lemma  \ref{lem:chi-sigma-diff}, for any given $V \in \V$, we have
\small
\begin{align*}
 &| \sigma_{\widehat{\Pp}^{\mathrm{c}}_{s,a}} (V) - \sigma_{{\Pp}^{\mathrm{c}}_{s,a}} (V)  | \leq  \max_{\mu: 0 \leq \mu \leq V} | \sqrt{c_r \text{Var}_{\widehat{P}_{s,a}}(V-\mu)} - \sqrt{c_r \text{Var}_{P^o_{s,a}}(V-\mu)}| + \max_{\mu: 0 \leq \mu \leq V} | \widehat{P}_{s,a} (V-\mu) - P^o_{s,a} (V-\mu)|. 
\end{align*}
\normalsize
By a simple variable substitution, we get
\begin{align*}
 &| \sigma_{\widehat{\Pp}^{\mathrm{c}}_{s,a}} (V) - \sigma_{{\Pp}^{\mathrm{c}}_{s,a}} (V)  | \leq  \max_{\mu \in \V_{+}} ~ \max_{s,a}~ | \sqrt{c_r \text{Var}_{\widehat{P}_{s,a}} \mu } - \sqrt{c_r \text{Var}_{P^o_{s,a}} \mu}|+ \max_{\mu \in \V} ~ \max_{s,a}~ | \widehat{P}_{s,a} \mu - P^o_{s,a}\mu|,
\end{align*}
which will give
\begin{align}
\label{eq:chi-pf-step-1}
\max_{V \in \V} ~ \max_{s,a} | \sigma_{\widehat{\Pp}^{\mathrm{c}}_{s,a}} (V) - \sigma_{{\Pp}^{\mathrm{c}}_{s,a}} (V) | \leq  \max_{\mu \in \V_{+}} ~ \max_{s,a}~ | \sqrt{c_r \text{Var}_{\widehat{P}_{s,a}} \mu } - \sqrt{c_r \text{Var}_{P^o_{s,a}} \mu}|+ \max_{\mu \in \V} ~ \max_{s,a}~ | \widehat{P}_{s,a} \mu - P^o_{s,a}\mu|,
\end{align}
where $\V_{+} =  \{V \in \R^{|\Ss|}_{+}: \norm{V} \leq 1/(1 - \gamma)\}$.

We will first bound the second term on the RHS of \eqref{eq:chi-pf-step-1}. From the proof of Lemma \ref{lem:covering-hoeffdings},  for any $\eta, \delta \in (0, 1)$, we get
 \begin{align}
\label{eq:chi-pf-step-2}
\max_{\mu\in \V} \max_{s,a} | \widehat{P}_{s,a} \mu - P^o_{s,a} \mu| &\leq \frac{1}{1-\gamma} \sqrt{\frac{|\Ss|\log(12|\Ss||\Aa|/(\delta\eta(1-\gamma)))}{2N}} + 2\eta,
\end{align} 
with probability greater than $1-\delta/2$. 

Now, we will focus  on the first term on the RHS of \eqref{eq:chi-pf-step-1}. Fix a $\mu \in \V_{+}$. Consider a minimal $\eta$-cover $\N_{\V_{+}}(\eta)$ of the set $\V_{+}$. By definition, there exists $\mu' \in \N_{\V_{+}}(\eta)$ such that $\norm{\mu - \mu'} \leq \eta$.  Now, following the same step as in the proof of Lemma \ref{lem:covering-hoeffdings}, we get
\begin{align*}
&| \sqrt{ \text{Var}_{\widehat{P}_{s,a}} \mu } - \sqrt{ \text{Var}_{P^o_{s,a}} \mu}| \leq    | \sqrt{ \text{Var}_{\widehat{P}_{s,a}} \mu } - \sqrt{ \text{Var}_{\widehat{P}_{s,a}} \mu'}|  +  | \sqrt{ \text{Var}_{\widehat{P}_{s,a}} \mu' } - \sqrt{ \text{Var}_{P^o_{s,a}} \mu'}|  + | \sqrt{ \text{Var}_{P^o_{s,a}} \mu } - \sqrt{ \text{Var}_{P^o_{s,a}} \mu'}| \\
&\stackrel{(a)}{\leq} | \sqrt{ \text{Var}_{\widehat{P}_{s,a}} \mu' } - \sqrt{ \text{Var}_{P^o_{s,a}} \mu'}|  +   \sqrt{| \text{Var}_{\widehat{P}_{s,a}} \mu - \text{Var}_{\widehat{P}_{s,a}} \mu' |}  +  \sqrt{| \text{Var}_{P^o_{s,a}} \mu - \text{Var}_{P^o_{s,a}} \mu'|} \\
&\stackrel{(b)}{\leq} | \sqrt{ \text{Var}_{\widehat{P}_{s,a}} \mu' } - \sqrt{ \text{Var}_{P^o_{s,a}} \mu'}|  +   \sqrt{| \widehat{P}_{s,a} (\mu^2 - \mu'^2) |} +   \sqrt{| (\widehat{P}_{s,a} \mu)^2 - (\widehat{P}_{s,a} \mu')^2) |}  + \\ &\hspace{7cm}   \sqrt{| P^o_{s,a} (\mu^2 - \mu'^2) |} +   \sqrt{| (P^o_{s,a} \mu)^2 - (P^o_{s,a} \mu')^2) |} \\
&\stackrel{(c)}{\leq} | \sqrt{ \text{Var}_{\widehat{P}_{s,a}} \mu' } - \sqrt{ \text{Var}_{P^o_{s,a}} \mu'}|  +   \sqrt{\frac{32\eta}{1-\gamma}} \\
&\leq \sup_{\mu'\in \N_{\V_{+}}(\eta)} \max_{s,a}| \sqrt{ \text{Var}_{\widehat{P}_{s,a}} \mu' } - \sqrt{ \text{Var}_{P^o_{s,a}} \mu'}|  +   \sqrt{\frac{32\eta}{1-\gamma}}
\end{align*} where $(a)$ follows from the fact $|\sqrt{x}-\sqrt{y}| \leq \sqrt{|x-y|}$ for all $x,y\in\R_+$, $(b)$ follows from the fact $|\sqrt{x+y}| \leq \sqrt{x}+\sqrt{y}$ for all $x,y\in\R_+$, and $(c)$ follows by using the fact $x^{2} - y^{2} =( x+y) (x-y)$, $\|\mu\| \leq 1/(1-\gamma)$, and  $\|\mu'\| \leq 1/(1-\gamma)$  with H$\ddot{\text{o}}$lder's inequality. Now, taking max on both sides with respect to $\mu$ and $(s,a)$ we get
\begin{align}
\sup_{\mu\in \V_{+}} \max_{s,a} |\sqrt{ \text{Var}_{\widehat{P}_{s,a}} \mu } - \sqrt{ \text{Var}_{P^o_{s,a}} \mu}| &\leq \sup_{\mu'\in \N_{\V_{+}}(\eta)} \max_{s,a}| \sqrt{ \text{Var}_{\widehat{P}_{s,a}} \mu' } - \sqrt{ \text{Var}_{P^o_{s,a}} \mu'}|  +   \sqrt{\frac{32\eta}{1-\gamma}} \nonumber \\
&\stackrel{(d)}{\leq} \frac{1}{1-\gamma} \sqrt{\frac{2\log(4|\Ss||\Aa| |\N_{\V_{+}}(\eta) |/\delta)}{N}} + \sqrt{\frac{32\eta}{1-\gamma}} \nonumber \\
&\stackrel{(e)}{\leq} \frac{1}{1-\gamma} \sqrt{\frac{2|\Ss|\log(12|\Ss||\Aa|/(\delta\eta(1-\gamma)))}{N}} + \sqrt{\frac{32\eta}{1-\gamma}}, \label{eq:chi-pf-step-3}
	\end{align} 
with probability at least $1-\delta/2$. Here,  $(d)$ follows from Lemma \ref{lem:variance_concentration} and the union bound and $(e)$ from Lemma \ref{por:covering_num}.

Applying \eqref{eq:chi-pf-step-2} and \eqref{eq:chi-pf-step-3} in \eqref{eq:chi-pf-step-1}, we get 
\begin{align*}
\max_{V \in \V} ~ \max_{s,a} ~| \sigma_{\widehat{\Pp}^{\mathrm{c}}_{s,a}} (V) - \sigma_{{\Pp}^{\mathrm{c}}_{s,a}} (V) | &\leq \frac{1}{1-\gamma} \sqrt{\frac{2 c_r |\Ss|\log(12|\Ss||\Aa|/(\delta\eta(1-\gamma)))}{N}} + \sqrt{\frac{32\eta c_r}{1-\gamma}}\\
&\hspace{1cm}  +  \frac{1}{1-\gamma} \sqrt{\frac{|\Ss|\log(12|\Ss||\Aa|/(\delta\eta(1-\gamma)))}{2N}} + 2\eta,
\end{align*}
with probability greater than $1 -\delta$. This completes the proof. 
\end{proof}

\begin{proof}[\textbf{Proof of Theorem \ref{thm:revi-Chi-guarantee}}] The basic steps of the proof is similar to that of Theorem \ref{thm:revi-TV-guarantee}. So, we present only the important steps.  

Following the same steps as given before \eqref{eq:tv-term1-bound} and using Proposition  \ref{lem:chi-sigma-diff-bound-uniform}, we get, with probability greater than $1-\delta$, 
\begin{align}
\label{eq:chi-term1-bound}
\|V^*-\widehat{V}^{*}\| &\leq \frac{\gamma}{(1-\gamma)} {C}^{\mathrm{c}}_{u}(N,\eta, \delta)
\end{align}
Similarly, following the steps as given before  \eqref{eq:tv-term2-bound}, we get 
\begin{align}
 \label{eq:chi-term2-bound}
 \|\widehat{V}^{\pi_{k}} -  \widehat{V}^*\|  \leq \frac{2\gamma^{k+1}}{(1-\gamma)^2}.
 \end{align}
In the same vein, following the steps as given before \eqref{eq:tv-term3-bound} and using Proposition  \ref{lem:chi-sigma-diff-bound-uniform}, we get, with probability greater than $1-\delta$, 
\begin{align}
\label{eq:chi-term3-bound}
\|V^{\pi_k}-\widehat{V}^{\pi_k}\| \leq \frac{\gamma}{(1-\gamma)} {C}^{\mathrm{c}}_{u}(N,\eta, \delta).
\end{align}
Using \eqref{eq:chi-term1-bound} -  \eqref{eq:chi-term3-bound}, similar to \eqref{eq:tv-thm-pf-st21}, we get, with probability greater than $1-2\delta$, 
\begin{align}
 \label{eq:chi-thm-pf-st21}
 \| V^* - V^{\pi_k} \| \leq \frac{2\gamma^{k+1}}{(1-\gamma)^2} +  \frac{2 \gamma}{(1-\gamma)} {C}^{\mathrm{c}}_{u}(N,\eta, \delta).
 \end{align}
 

 Using the value of ${C}^{\mathrm{c}}_{u}(N,\eta, \delta)$ as given in  Proposition \ref{lem:chi-sigma-diff-bound-uniform}, we get, with probability greater than $1-2\delta$,
\begin{align*}
\| V^* - V^{\pi_k} \| \leq &\frac{2\gamma^{k+1}}{(1-\gamma)^2} + \frac{8\gamma{\sqrt{2\eta c_r}}}{(1-\gamma)^{3/2}} + \frac{4\gamma\eta}{1-\gamma}
+ \frac{2\gamma}{(1-\gamma)^2} \sqrt{\frac{(2 c_r+1) |\Ss|\log(12|\Ss||\Aa|/(\delta\eta(1-\gamma)))}{N}} .
\end{align*}
We can now choose $k, \epsilon, \eta$ to make each of the term on the RHS of the above inequality small. In particular, we select $\epsilon\in(0,\min\{16\gamma /(1-\gamma),32\gamma \sqrt{2c_r} /(1-\gamma)^{3/2}  \})$ and $\eta = \min \{ \epsilon (1-\gamma)/(16\gamma) , \epsilon^2 (1-\gamma)^3/(2048 c_r \gamma^2)  \}$. Note that this choice also ensure $\eta \in (0, 1)$. Now, by choosing
\begin{align}
\label{eq:K0-chi-modified}
k &\geq K_{0} = \frac{1}{ \log(1/\gamma)} \cdot \log(\frac{8\gamma}{\epsilon (1-\gamma)^2}), \\
\label{eq:N-chi-modified}
N &\geq N^{\mathrm{c}} = 
	\frac{ 64 \gamma^2}{(1-\gamma)^4} \cdot \frac{(2c_r + 1)|\Ss|\log(12|\Ss||\Aa|/(\delta\eta(1-\gamma)))}{\epsilon^2},
\end{align}
we will get $\| V^* - V^{\pi_k} \| \leq \epsilon$ with probability at least $1-2\delta$.
\end{proof}

\subsection{Proof of Theorem \ref{thm:revi-KL-guarantee} }

We  state a  result  from \citep{zhou2021finite} that will be useful in the proof of Theorem \ref{thm:revi-KL-guarantee}.

\begin{lemma}[\text{\citep[Lemma 4]{zhou2021finite} }]
\label{lem:zhou-lambda-*-lemma-4}
    Fix any $\delta\in(0,1)$. Let $X\sim P$ be a bounded random variable with $X\in [0,M]$ and let $P_N$ denote the empirical distribution of $P$ with $N$ samples. For $t > 0$, for any \[ \lambda^* \in \argmax_{\lambda\geq 0} \left\{ -\lambda \log(\E_{P}[\exp(-X/\lambda)]) - \lambda t  \right\}, \] 
    (1) $\lambda^* = 0$. Furthermore, let the support of $X$ be finite. Then there exists a problem dependent constant \[ N'(\delta, t, P) := \max \{ \log(2/\delta) / \log(1/(1-\min_{x\in \mathrm{supp}(X)} P(X=x))) ~,~ 2M^2 \log(4/\delta) / (P(X= \mathrm{ess}\inf X) - \exp(-t))^2 \} ,\] such that for $N\geq N'(\delta, t, P)$ we have, with probability at least $1-\delta$, \[ 0 \in \argmax_{\lambda\geq 0} \left\{ -\lambda \log(\E_{P_N}[\exp(-X/\lambda)]) - \lambda t  \right\}. \]
    (2) $\lambda^* > 0$. Then there exists a problem dependent constant \[ N''(\delta, t, P) := \max_{\lambda \in \{ \underline{\lambda}, \lambda^*, M/t \} } \frac{8M^2 \exp(2M/\lambda)}{\tau^2} \log(6/\delta) ,\] 
    where $\underline{\lambda} = \lambda^*/2 >0$ (independent of $N$) and \begin{align*}
        \tau &= \min \{ \underline{\lambda} \log(\E_{P}[\exp(-X/\underline{\lambda})]) + \underline{\lambda} t , (M/t) \log(\E_{P}[\exp(-t X/M)]) + M \} \\
        &\hspace{7cm}- (\lambda^* \log(\E_{P}[\exp(-X/\lambda^*)]) + \lambda^* t ) > 0,
    \end{align*}
    such that for $N\geq N''(\delta, t, P)$, with probability at least $1-\delta$,  there exists a \[ \widehat{\lambda}^* \in \argmax_{\lambda\geq 0} \left\{ -\lambda \log(\E_{P_N}[\exp(-X/\lambda)]) - \lambda t  \right\}, \] such that $\lambda^*, \widehat{\lambda}^* \in [\underline{\lambda}, M/t]$.
\end{lemma}

We now prove the following result. 

\begin{lemma}
\label{lem:kl-characterization-sigma-diff}
For any $(s, a) \in \Ss \times \Aa$ and for any $V \in \mathbb{R}^{|\Ss|}$  with $\norm{V} \leq 1/(1-\gamma)$, 
\begin{align}
\label{eq:kl-sigma-diff}
| \sigma_{\widehat{\Pp}^{\mathrm{kl}}_{s,a}} (V) - \sigma_{{\Pp}^{\mathrm{kl}}_{s,a}} (V) | \leq \frac{\exp({1}/{\lambda_{\mathrm{kl}}(1-\gamma)})}{c_{r} (1-\gamma) } ~ \max_{\lambda\in[\lambda_{\mathrm{kl}},\frac{1}{c_{r}(1-\gamma)}]} ~ |(P^o_{s,a}- \widehat{P}_{s,a}) \exp(-V/\lambda)| 
\end{align}
holds with probability at least $1-\delta/(2|\Ss||\Aa|)$ for $N\geq \max \{ N'(\delta/(4|\Ss||\Aa|), c_r, P^o_{s,a}), N''(\delta/(4|\Ss||\Aa|), c_r, P^o_{s,a}) \},$ where both $N', N''$ are defined as in  Lemma \ref{lem:zhou-lambda-*-lemma-4}.
\end{lemma}
\begin{proof}
Fix any $(s, a)$ pair.  From \cite[Lemma 4.1]{iyengar2005robust}, we have 
\begin{align} 
\label{eq:sigma-KL}
&\sigma_{{\Pp}^{\mathrm{kl}}_{s,a}} (V) = \max_{\lambda \geq 0}  ~( -c_r \lambda  - \lambda \log (P^{o}_{s,a} \exp(-V/\lambda)) ),   \quad \sigma_{\widehat{\Pp}^{\mathrm{kl}}_{s,a}} (V) = \max_{\lambda \geq 0}  ~( -c_r \lambda  - \lambda \log (\widehat{P}_{s,a} \exp(-V/\lambda)) ),
\end{align} 
where $\exp(-V/\lambda)$ is an element-wise exponential function. It is straight forward to show that $( -c_r \lambda  - \lambda \log (P^{o}_{s,a} \exp(-V/\lambda)) )$ is a concave function in $\lambda$. So, there exists an optimal solution $\lambda^{*}$. Similarly, let  $\widehat{\lambda}^{*}$ be the optimal solution of the second problem above. 

We can now give an upperbound for $\lambda^{*}, \widehat{\lambda}^{*}$ as follows:  Since $\sigma_{{\Pp}^{\mathrm{kl}}_{s,a}} (V)  \geq 0$,  we have
\begin{align*}
0 &\leq -c_r \lambda^{*}  - \lambda^{*} \log (P^{o}_{s,a}  \exp(-V/\lambda^{*} )) \stackrel{(a)}{\leq} -c_r \lambda^{*}   - \lambda^{*}  \log ( \exp(-1/(\lambda^{*} (1-\gamma)))) \leq -c_r \lambda^{*}  +{1}/{(1-\gamma)},
\end{align*}
from which we can conclude that $\lambda^{*}  \leq 1/(c_{r} (1-\gamma))$.  Same argument applies for the case of $\widehat{\lambda}^{*}$. 


From \citep[Appendix C]{nilim2005robust} it follows that whenever the maximizer $\lambda^*$ is $0$ ( $\widehat{\lambda}^*$ is $0$), we have $\sigma_{{\Pp}^{\mathrm{kl}}_{s,a}} (V) = V_{\min}$ ( $\sigma_{\widehat{\Pp}^{\mathrm{kl}}_{s,a}} (V) = V_{\min}$) where $V_{\min}=\min_{j\in\Ss} V(j)$. We include this part in detail  for completeness. 
\begin{align*}
    \lim_{\lambda\downarrow 0}-c_r \lambda & - \lambda \log (P^{o}_{s,a} \exp(-V/\lambda))   = \lim_{\lambda\downarrow 0} -c_r \lambda  - \lambda \log (\exp(-V_{\min}/\lambda) \sum_{s'} P^{o}_{s,a}(s') \exp((V_{\min}-V(s'))/\lambda)) \\
    &= \lim_{\lambda\downarrow 0} V_{\min} -c_r \lambda  - \lambda \log ( \sum_{s'} P^{o}_{s,a}(s') \exp((V_{\min}-V(s'))/\lambda)) \\
    &= \lim_{\lambda\downarrow 0} V_{\min} -c_r \lambda  - \lambda \log ( \sum_{s':V(s')=V_{\min}} P^{o}_{s,a}(s') + \sum_{s':V(s')>V_{\min}} P^{o}_{s,a}(s') \exp((V_{\min}-V(s'))/\lambda)) \\
    &\stackrel{(a)}{=} \lim_{\lambda\downarrow 0} V_{\min} -c_r \lambda  - \lambda \log ( \sum_{s':V(s')=V_{\min}} P^{o}_{s,a}(s') + \cO(\exp(-t/\lambda))) \\ 
    &\stackrel{(b)}{=}  \lim_{\lambda\downarrow 0} V_{\min} -c_r \lambda  - \lambda \log ( \sum_{s':V(s')=V_{\min}} P^{o}_{s,a}(s'))  - \lambda \log(1+ \cO(\exp(-t/\lambda))) \\
    &\stackrel{(c)}{=}  \lim_{\lambda\downarrow 0} V_{\min} -\lambda (c_r +  \log ( \sum_{s':V(s')=V_{\min}} P^{o}_{s,a}(s')) ) - \cO(\lambda \exp(-t/\lambda)) = V_{\min},
    \end{align*} where $(a)$ follows by taking $t=\min_{s':V(s')>V_{\min}}V(s') - V_{\min} > 0$, and $(b)$ and $(c)$ follows from the Taylor series expansion. Thus when $\lambda^*$ is $0$, we have $\sigma_{{\Pp}^{\mathrm{kl}}_{s,a}} (V) = V_{\min}$. A similar argument applies for $\sigma_{\widehat{\Pp}^{\mathrm{kl}}_{s,a}} (V)$.
    
Now consider the case when $\lambda^*=0$. From Lemma \ref{lem:zhou-lambda-*-lemma-4}, it follows that, with probability at least $1-\delta/(4|\Ss||\Aa|)$,  $\widehat{\lambda}^*=0$ for $N\geq N'(\delta/(4|\Ss||\Aa|), c_r, P^o_{s,a}),$ where $N'$ is defined in  Lemma \ref{lem:zhou-lambda-*-lemma-4}.    Thus, whenever $\lambda^*=0$, we have $| \sigma_{\widehat{\Pp}_{s,a}} (V) - \sigma_{{\Pp}_{s,a}} (V)  | = |V_{\min} - V_{\min}| = 0$, with probability at least $1-\delta/(4|\Ss||\Aa|)$. Thus having resolving this trivial case, we now focus on the case  when $\lambda^*>0$.

Consider the case when $\lambda^*>0$. Let $\lambda_{\mathrm{kl}} := \lambda^*/2 >0$ (dependent on $P^o_{s,a}, V,$ and $c_r$ but independent of $N$). Again from Lemma \ref{lem:zhou-lambda-*-lemma-4}, if $\lambda^* \in [\lambda_{\mathrm{kl}},{1}/{(c_r(1-\gamma))}]$, then with probability at least $1-\delta/(4|\Ss||\Aa|)$ we have $\widehat{\lambda}^*\in [\lambda_{\mathrm{kl}},{1}/{(c_r(1-\gamma))}]$ for $N\geq N''(\delta/(4|\Ss||\Aa|), c_r, P^o_{s,a})$, where $N''$ is defined in  Lemma \ref{lem:zhou-lambda-*-lemma-4}.

{ From these arguments, it is clear that we can restrict the optimization problem \eqref{eq:sigma-KL} to the set $\lambda \in [\lambda_{\mathrm{kl}},  1/(c_{r} (1-\gamma))$}.  Using this, with the additional fact that $|\max_{x} f(x)-\max_{x} g(x)| \leq \max_{x} |f(x) - g(x)|$, we get
\begin{align}
\label{eq:sigma-Kl-pf-1}
&| \sigma_{\widehat{\Pp}^{\mathrm{kl}}_{s,a}} (V) - \sigma_{{\Pp}^{\mathrm{kl}}_{s,a}} (V)  |  \leq  \max_{\lambda\in[\lambda_{\mathrm{kl}},\frac{1}{c_{r}(1-\gamma)}]} | \lambda \log (\frac{ \widehat{P}_{s,a} \exp(-V/\lambda) }{P^o_{s,a} \exp(-V/\lambda)} )|. 
\end{align}
Now, 
\begin{align}
\abs{ \log (\frac{ \widehat{P}_{s,a} \exp(-V/\lambda) }{P^o_{s,a} \exp(-V/\lambda)} )} =  \abs{\log (1 + \frac{(\widehat{P}_{s,a} - P^o_{s,a}) \exp(-V/\lambda) }{P^o_{s,a}\exp(-V/\lambda) } )} &\leq    \frac{|(P^o_{s,a}- \widehat{P}_{s,a}) \exp(-V/\lambda)| }{| P^o_{s,a} \exp(-V/\lambda) |}  \nonumber \\
\label{eq:sigma-Kl-pf-2}
&\stackrel{(d)}{\leq} \frac{|(P^o_{s,a}- \widehat{P}_{s,a}) \exp(-V/\lambda)| }{\exp(\frac{-1}{\lambda_{\mathrm{kl}}(1-\gamma)})},
\end{align}
where $(d)$ follows since $\lambda\geq  \lambda_{\mathrm{kl}}$ and $\norm{V} \leq 1/(1-\gamma)$. Using \eqref{eq:sigma-Kl-pf-2} in \eqref{eq:sigma-Kl-pf-1} along with the fact that $\lambda \leq 1/(c_{r} (1-\gamma))$, we get the desired result. 
\end{proof}

\begin{proof}[\textbf{Proof of Theorem \ref{thm:revi-KL-guarantee}}] The basic steps of the proof is similar to that of Theorem \ref{thm:revi-TV-guarantee}. So, we present only the important steps.  

Following the same steps as given before \eqref{eq:tv-term1-incomplete-bound} and \eqref{eq:tv-term3-incomplete-bound} , we get
\begin{align}
\|V^*-\widehat{V}^{*}\| + \|V^{\pi_k}-\widehat{V}^{\pi_k}\| &\leq \frac{2\gamma}{(1-\gamma)} ~ \max_{V \in \V} ~ \max_{s,a}  ~| \sigma_{\widehat{\Pp}^{\mathrm{kl}}_{s,a}} (V) - \sigma_{{\Pp}^{\mathrm{kl}}_{s,a}} (V) | . \label{eq:kl-term1-term3-incomplete-bound}
\end{align}
Similarly, following the steps as given before  \eqref{eq:tv-term2-bound}, we get 
\begin{align}
 \label{eq:kl-term2-bound}
 \|\widehat{V}^{\pi_{k}} -  \widehat{V}^*\|  \leq \frac{2\gamma^{k+1}}{(1-\gamma)^2}.
 \end{align}
 
 Using Lemma  \ref{lem:kl-characterization-sigma-diff} in \eqref{eq:kl-term1-term3-incomplete-bound}, we get
\begin{align}
\label{eq:kl-term1-term3-incomplete-bound-2}
\|V^*-\widehat{V}^{*}\| + \|V^{\pi_k}-\widehat{V}^{\pi_k}\| \leq \frac{2\gamma}{(1-\gamma)} \frac{\exp({1}/{(\lambda_{\mathrm{kl}}(1-\gamma))})}{c_{r} (1-\gamma) } ~ \max_{s,a} ~ \max_{V \in \V}  ~\max_{\lambda\in[\lambda_{\mathrm{kl}},\frac{1}{c_{r}(1-\gamma)}]} ~ |(P^o_{s,a}- \widehat{P}_{s,a}) \exp(-V/\lambda)| .
\end{align}
We now bound the max term in \eqref{eq:kl-term1-term3-incomplete-bound-2}. We reparameterize $1/\lambda$ as $\theta$ and consider the set   $\Theta = [c_{r}(1-\gamma),\frac{1}{\lambda_{\mathrm{kl}}}]$. Also, consider the minimal $\eta$-cover  $\N_\Theta(\eta)$ of $\Theta$ and fix a $V\in\V$. Then, for any given $\theta \in \Theta$, there exits a $\theta' \in \N_\Theta(\eta)$ such that $|\theta - \theta'| \leq \eta$. Now, for this particular $\theta, \theta'$,
\begin{align*}
&| ( P^o_{s,a} - \widehat{P}_{s,a}) \exp(-V\theta)| =   | (\widehat{P}_{s,a}  - P^o_{s,a}) ( \exp(-V\theta') \circ \exp(-V(\theta-\theta'))  |  \\
&\stackrel{(c)}{\leq}  | (\widehat{P}_{s,a}  - P^o_{s,a}) \exp(-V\theta')| \exp(\eta/(1-\gamma)) \leq \max_{s,a} \max_{\theta'\in \N_\Theta(\eta)} | (\widehat{P}_{s,a}  - P^o_{s,a}) \exp(-V\theta')| \exp(\eta/(1-\gamma)), 
\end{align*} 
where $(c)$ follows because $V$ is non-negative and  $\norm{V}\leq1/(1-\gamma)$. Now consider a minimal $\eta$-cover $\N_{\V}(\eta)$ of the set $\V$. By definition, there exists $V' \in \N_{\V}(\eta)$ such that $\norm{V - V'} \leq \eta$. So, we get
\begin{align*}
| ( P^o_{s,a} - \widehat{P}_{s,a}) \exp(-V\theta)| &\leq  | (\widehat{P}_{s,a}  - P^o_{s,a}) \exp(-V\theta')| \exp(\eta/(1-\gamma)) \\&= | (\widehat{P}_{s,a}  - P^o_{s,a}) (\exp(-V'\theta') \circ \exp(\theta'(V'-V)))| \exp(\eta/(1-\gamma)) \\
&\stackrel{(d)}{\leq} | (\widehat{P}_{s,a}  - P^o_{s,a}) (\exp(-V'\theta'))| \exp(\eta/(1-\gamma)) \exp(\eta/\lambda_{\mathrm{kl}}) \\
&{\leq} \max_{s,a} \max_{V'\in\V}  \max_{\theta'\in \N_\Theta(\eta)}  | (\widehat{P}_{s,a}  - P^o_{s,a}) (\exp(-V'\theta'))| \exp(\eta/(1-\gamma)) \exp(\eta/\lambda_{\mathrm{kl}})
\end{align*} where $(d)$ follows because $\theta'\in\N_\Theta(\eta)\subseteq \Theta$.
Now, taking maximum on both sides with respect to $(s,a)$, $\theta$, and $V$, we get
\begin{align}
\max_{s,a} \max_{\theta\in \Theta} \max_{V\in\V}  | (\widehat{P}_{s,a}  - P^o_{s,a}) \exp(-V\theta)| &\leq \exp(\eta/(1-\gamma)) \exp(\eta/\lambda_{\mathrm{kl}}) \max_{s,a} \max_{V'\in\V} \max_{\theta'\in \N_\Theta(\eta)}  | (\widehat{P}_{s,a}  - P^o_{s,a}) \exp(-V'\theta')|  \nonumber\\
\nonumber
&\stackrel{(e)}{\leq} \exp(\eta/(1-\gamma))
\exp(\eta/\lambda_{\mathrm{kl}})\sqrt{\frac{\log(2|\Ss||\Aa||\N_\Theta(\eta)||\N_\V(\eta)|/\delta)}{2N}}\\
\label{eq:Kl-pf-step3}&\stackrel{(f)}{\leq} \exp(\eta/(1-\gamma))\exp(\eta/\lambda_{\mathrm{kl}}) \sqrt{\frac{|\Ss|\log(18|\Ss||\Aa|/(\delta\eta^2(1-\gamma)\lambda_{\mathrm{kl}}))}{2N}}
\end{align}
with probability greater than $1-\delta$. Here, $(e)$ follows from    Lemma \ref{lem:hoeffding_concentration} with a union bound accounting for $|\N_{\Theta}(\eta)|$, $|\N_{\V}(\eta)|$ and the fact that $\|\exp(-V'\theta')\| \leq 1$, and $(f)$ follows from Lemmas \ref{por:covering_num} and \ref{por:covering_num_real_line}.

Using \eqref{eq:kl-term1-term3-incomplete-bound} -  \eqref{eq:Kl-pf-step3}, we get, with probability greater than $1-\delta$,
\begin{align*}
\| V^* - V^{\pi_k} \| \leq &\frac{2\gamma^{k+1}}{(1-\gamma)^2} + \\& \frac{2\gamma}{(1-\gamma)} \frac{\exp({1}/{(\lambda_{\mathrm{kl}}(1-\gamma))})}{c_{r} (1-\gamma) } \exp(\eta/(1-\gamma))\exp(\eta/\lambda_{\mathrm{kl}}) \sqrt{\frac{|\Ss|\log(18|\Ss||\Aa|/(\delta\eta^2(1-\gamma)\lambda_{\mathrm{kl}}))}{2N}} .
\end{align*}
We can now choose $k, \epsilon, \eta$ to make each of the term on the RHS of the above inequality small. In particular, choosing $\eta= 1$,  $\epsilon\in(0,1/(1-\gamma))$, and $k, N$ satisfying the conditions
\begin{align*}
	    &k \geq K_{0}  = \frac{1}{ \log(1/\gamma)} \cdot \log(\frac{4}{\epsilon (1-\gamma)^2}) \qquad\textnormal{and} \\
	    &N \geq N^{\mathrm{kl}} = \max \bigg\{ \max_{s,a} N'(\delta/(4|\Ss||\Aa|), c_r, P^o_{s,a}), ~ \max_{s,a} N''(\delta/(4|\Ss||\Aa|), c_r, P^o_{s,a}), \\
	    &\hspace{4cm}\frac{8\gamma^2|\Ss|}{c_r^2(1-\gamma)^4 \epsilon^{2}}  \exp(\frac{4+2\lambda_{\mathrm{kl}}}{\lambda_{\mathrm{kl}}(1-\gamma)}) \log(\frac{18 |\Ss|  |\Aa|}{\delta\lambda_{\mathrm{kl}}(1-\gamma)}) \bigg\},
\end{align*} 
we get $\| V^* - V^{\pi_k} \|\leq \epsilon$ with probability greater than $1-\delta$. 
\end{proof}

\subsection{Proof of Theorem \ref{thm:non-robust-policy-suboptimality}}

\begin{proof}
    We consider the deterministic  MDP $(\Ss, \Aa, r, P^o, \gamma)$ shown in Fig.\ref{fig:chain-mdp-nominal} to be the nominal model. We fix $\gamma\in(0.01,1]$ and $s_1=0$. The state space is $\Ss = \{0,1\}$ and action space is $\Aa = \{a_{l}, a_{r}\},$ where $a_{l}$ denotes `move left' and $a_{r}$ denotes `move right' action. Reward for state $1$ and action $a_r$ pair is $r(1, a_{r}) = 1$, for state $0$ and action $a_r$ pair is $r(0, a_{r}) = -100\gamma/99$, and the reward is $0$ for all other $(s,a)$. Transition function $P^{o}$ is deterministic, as indicated by the arrows. 
\begin{figure*}[ht]
    \begin{minipage}{.47\textwidth}
	\centering
	\includegraphics[width=0.5\linewidth]{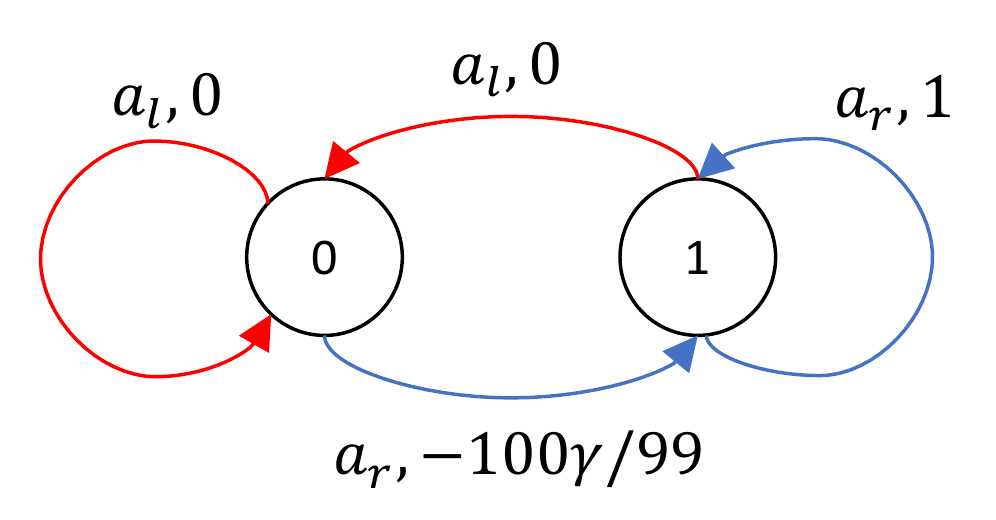}
	\captionof{figure}{Transitions and rewards corresponding to the nominal model $P^o$. The states $\{0,1\}$ are given inside the circles, and the actions $\{a_l,a_r\}$ and  associated  rewards are given on the corresponding  transitions.}
	\label{fig:chain-mdp-nominal}
	\end{minipage}
	\hspace{0.5cm}
	\begin{minipage}{.47\textwidth}
	\centering
	\includegraphics[width=0.5\linewidth]{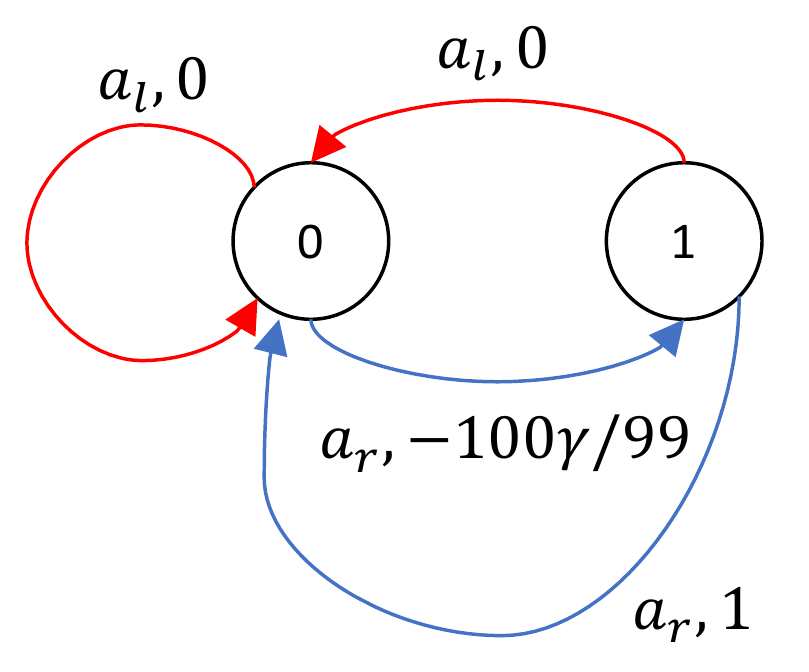}
	\captionof{figure}{Transitions and rewards corresponding to the model $P'$.}
	\label{fig:chain-mdp-p}
	\end{minipage}
\end{figure*}

Similarly, we consider another deterministic model $P'$, as shown in Fig.\ref{fig:chain-mdp-p}. We consider the set $\Pp=\{P^o,P'\}$. 

    It is straight forward to show that taking action $a_r$ in any state  is the optimal non-robust policy $\pi^{o}$ corresponding to the nominal model $P^{o}$. This is obvious if for state $s=1$. For $s = 0$, notice that taking action $a_{l}$ will give a value zero and taking action $a_{r}$ will give a value  $\frac{\gamma}{1-\gamma} - \frac{100\gamma}{99}$. Since $\gamma > 0.01$, taking action $a_{r}$ will give a positive value and hence  is optimal. So, we get 
    \[V_{{\pi}^{o},P^o}(0) = \frac{\gamma}{1-\gamma} - \frac{100\gamma}{99}.\]
    

    We can now compute $V_{{\pi}^{o},P'}(0)$ using the recursive equation
    \[ V_{{\pi}^{o},P'}(0) = - \frac{100\gamma}{99} + \gamma + \gamma^2 V_{{\pi}^{o},P'}(0). \] 
    Solving this, we get $V_{{\pi}^{o},P'}(0) = -\gamma/(99(1-\gamma^2))$.

    Now the robust value of $\pi^{o}$ is given by
    \begin{align*}
        V^{{\pi}^{o}}(0) = \min\{V_{{\pi}^{o},P^o}(0),V_{{\pi}^{o},P'}(0)\} = -\gamma/(99(1-\gamma^2)). 
    \end{align*}
    
    We will now compute the optimal non-robust value from state $0$ of model $P'$. \begin{align*}
        \max_\pi V_{\pi,P'}(0) &= \max \{ V_{(\pi(0)=a_r,\pi(1)=a_r),P'}(0),~ V_{(\pi(0)=a_l,\pi(1)=a_l),P'}(0),~ \\
        &\hspace{2cm}V_{(\pi(0)=a_r,\pi(1)=a_l),P'}(0),~ V_{(\pi(0)=a_l,\pi(1)=a_r),P'}(0)\} \\
        &= \max \{ ~-\frac{\gamma}{99(1-\gamma^2)} ,~ 0,~ - \frac{100\gamma}{99(1+\gamma^2)} ,~0 \} = 0.
    \end{align*}

    Now, we find the optimal robust value $V^{*}(0)$. From the perfect duality result of robust MDP \cite[Theorem 1]{nilim2005robust}, we have
    \begin{align*}
        V^*(0) &= \min \{\max_\pi V_{\pi,P^o}(0), \max_\pi V_{\pi,P'}(0)\} = \min \{V_{{\pi}^{o},P^o}(0),  \max_\pi V_{\pi,P'}(0)\} = 0. 
    \end{align*}
We finally have \begin{align*}
        V^*(0) - V^{{\pi}^{o}}(0) = \frac{\gamma}{99(1-\gamma^2)} \geq \frac{\gamma}{198(1-\gamma)},
    \end{align*} where the inequality follows since $1+\gamma\leq2$. Thus, setting $c=\gamma/198$ and $\gamma_o=0.01$, completes the proof of this theorem.
\end{proof}

\end{document}